\documentclass[a4paper,UKenglish,cleveref, autoref, thm-restate]{lipics-v2021}

\usepackage{custom_macro}
\usepackage{array, makecell}
\renewcommand{\epsilon}{\varepsilon}

\title{Instance-Optimal Matrix Multiplicative Weight Update and Its Quantum Applications} 

\author{Weiyuan Gong}{School of Engineering and Applied Sciences, Harvard University, MA, USA}{wgong@g.harvard.edu}{https://orcid.org/0000-0002-6599-8110}{}

\author{Tongyang Li}{Center on Frontiers of Computing Studies and School of Computer Science, Peking University, Beijing, China}{tongyangli@pku.edu.cn}{https://orcid.org/0000-0002-0338-413X}{}

\author{Xinzhao Wang}{Center on Frontiers of Computing Studies and School of Computer Science, Peking University, Beijing, China}{wangxz@stu.pku.edu.cn}{https://orcid.org/0000-0003-2191-6117}{}

\author{Zhiyu Zhang}{Department of Statistics and Data Science, Carnegie Mellon University, PA, USA}{zhiyuzresearch@gmail.com}{https://orcid.org/0000-0001-9639-8196}{}

\authorrunning{W. Gong, T. Li, X. Wang, and Z. Zhang} 

\Copyright{Weiyuan Gong, Tongyang Li, Xinzhao Wang, and Zhiyu Zhang} 



\relatedversion{} 

\acknowledgements{We thank Shucheng Kang and Haoyu Han for insightful calculations and suggestions regarding the matrix trace inequality. We thank Sitan Chen, Sabee Grewal, Francesco Orabona, and Dong Yuan for helpful discussions.}

\nolinenumbers 

\ArticleNo{1}

\begin{document}

\maketitle

\begin{abstract}
The Matrix Multiplicative Weight Update (MMWU) is a seminal online learning algorithm with numerous applications. Applied to the matrix version of the Learning from Expert Advice (LEA) problem on the $d$-dimensional spectraplex, it is well known that MMWU achieves the minimax-optimal regret bound of $\mathcal{O}\rpar{\sqrt{T\log d}}$, where $T$ is the time horizon. In this paper, we present an improved algorithm achieving the instance-optimal regret bound of $\mathcal{O}\rpar{\sqrt{T\cdot S(X||d^{-1}I_d)}}$, where $X$ is the comparator in the regret, $I_d$ is the identity matrix, and $S(\cdot||\cdot)$ denotes the quantum relative entropy. Furthermore, our algorithm has the same computational complexity as MMWU, indicating that the improvement in the regret bound is ``free''. 

Technically, we first develop a general potential-based framework for matrix LEA, with MMWU being its special case induced by the standard exponential potential. Then, the crux of our analysis is a new ``one-sided'' Jensen's trace inequality built on a Laplace transform technique, which allows the application of general potential functions beyond exponential to matrix LEA. Our algorithm is finally induced by an optimal potential function from the vector LEA problem, based on the imaginary error function. 

Complementing the above, we provide a memory lower bound for matrix LEA, and explore the applications of our algorithm in quantum learning theory. We show that it outperforms the state of the art for learning quantum states corrupted by depolarization noise, random quantum states, and Gibbs states. In addition, applying our algorithm to linearized convex losses enables predicting nonlinear quantum properties, such as purity, quantum virtual cooling, and R\'{e}nyi-$2$ correlation.
\end{abstract}

\section{Introduction}

The \emph{Matrix Multiplicative Weight Update} (MMWU) algorithm~\cite{tsuda2005matrix} is a fundamental tool in online optimization, machine learning, and theoretical computer science, with applications in semidefinite programming~\cite{arora2007combinatorial}, spectral sparsifiers~\cite{allen2015spectral}, and quantum learning theory~\cite{aaronson2018online}. It is designed to solve the following matrix version of \emph{Learning from Expert Advice} (LEA), which is a classical online learning problem formulated as a repeated game. 

\begin{definition}[The matrix LEA problem]\label{def:matrix_lea}
Let the domain $\mathcal{X}$ be the $d$-dimensional \emph{spectraplex} $\Delta_{d\times d}$, the collection of all $d\times d$ Hermitian positive semidefinite (PSD) matrices with unit trace. At every time step $t$ of the total $T$ iterations, the learner chooses a prediction $X_t\in \mathcal{X}$. The adversary then supplies a Hermitian loss matrix $G_t$ satisfying $\norm{G_t}_\op\leq l$, and the learner suffers from the loss $\inner{G_t}{X_t}=\tr(G_tX_t)\in\R$, with $\inner{\cdot}{\cdot}$ denoting the matrix Frobenius inner product. The performance of the learner is measured by the \emph{regret}
\begin{align*}
\reg_T(X)\defeq \sum_{t=1}^T\inner{G_t}{X_t}-\sum_{t=1}^T\inner{G_t}{X},
\end{align*}
which quantifies the cumulative excess loss of the learner against any fixed benchmark prediction $X\in\mathcal{X}$, known as the \emph{comparator}.
\end{definition} 

For this problem, the MMWU algorithm (introduced in Section~\ref{sec:prelim_lea_mmwu}) guarantees the regret bound of $\mathcal{O}\rpar{\sqrt{T\log d}}$~\cite{warmuth2008randomized,warmuth2006online}, which is asymptotically optimal in the worst case of the comparator $X$~\cite{cesa1997use}. Based on this result, an extensive line of works has revolved around its applications in quantum information, such as the online learning of quantum states~\cite{aaronson2018online} and processes~\cite{bansal2025online,raza2024online}. However, two fundamental questions are left open.

\vspace{0.5em}\noindent \textbf{(A) Instance-optimal regret bound. }For MMWU, the regret bound of $\mathcal{O}\rpar{\sqrt{T\log d}}$ is a fixed non-adaptive quantity that does not depend on the comparator $X$. Consequently, it fails to improve on intuitively ``easier'' comparators, such as those with high entropy (i.e., close to the normalized identity matrix $d^{-1}I_d$, which is also known as the maximally mixed state in quantum information). Prior works in the standard vector setting of LEA suggest that this is due to the algorithm design rather than just the regret analysis \cite{chaudhuri2009parameter,luo2015achieving,orabona2016coin}. Therefore, it motivates the central question: 
\begin{center}
\emph{Is there a better algorithm retaining the computational complexity of MMWU, while achieving an \emph{instance-optimal} regret bound with respect to the comparator $X$?}
\end{center} 

\vspace{0.5em}\noindent \textbf{(B) Instance-dependent complexity of quantum learning theory. } A natural application of MMWU is the online learning of quantum states~\cite{aaronson2018online} and processes~\cite{raza2024online,bansal2025online}. This originates from a fundamental task in quantum information -- learning from quantum systems and their evolution. While fully recovering the description of an unknown quantum system is resource-consuming~\cite{haah2016sample,o2016efficient}, learning a set of properties of the unknown quantum system, also known as shadow tomography~\cite{aaronson2018shadow}, can be solved sample-efficiently. The matrix LEA problem serves as the key subroutine of \emph{online shadow tomography}. In addition, since no statistical assumption is required on the data-generating process, it can provide worst-case guarantees for learning from quantum data. 

Within this subfield, a growing research direction is to characterize the complexity of online learning various quantum states and processes, given certain prior knowledge on their algebraic structures~\cite{raza2024online,bansal2025online}. However, without assuming prior knowledge, an instance-dependent understanding of this question is lacking hitherto. It is thus natural to ask: 
\begin{center}
\emph{Can an improvement on MMWU translate to an instance-dependent characterization of the complexity in quantum online learning, without assuming prior knowledge?} 
\end{center}

This paper answers both of the above questions affirmatively. 

\subsection{Our results}

\subsubsection{Improving MMWU}\label{sec:intro_lea}

Regarding the matrix LEA problem (Definition~\ref{def:matrix_lea}), our main result is stated as follows. 
\begin{theorem}[Informal, see Theorem~\ref{thm:lea_main} and Remark~\ref{remark:computation}]\label{thm:lea_main_informal}
There exists an algorithm (presented in Section~\ref{sec:algorithm}) such that for all $T\geq 1$ and $X\in\calX$, 
\begin{equation}\label{eq:regret_ours_summary}
\reg_T(X)=\mathcal{O}\rpar{\sqrt{T\cdot S(X||d^{-1}I_d)}},
\end{equation}
where $S(X_1||X_2)\defeq\inner{X_1}{\log X_1-\log X_2}$ denotes the quantum relative entropy of $X_1$ relative to $X_2$, extending the KL divergence of probability vectors to density matrices. Furthermore, the time and memory complexity of this algorithm are identical to those of MMWU. 
\end{theorem}

Discussions are required to interpret this result. First, it is known that $S(X||d^{-1}I_d)\leq \log d$ for all $X\in\calX$, therefore Eq.~\eqref{eq:regret_ours_summary} is never worse than the $\calO\rpar{\sqrt{T\log d}}$ regret bound of MMWU. Given that, the value of Theorem~\ref{thm:lea_main_informal} is the substantial improvement over MMWU when the comparator $X$ is ``easy''. For example:
\begin{itemize}
    \item We obtain $\mathcal{O}(\sqrt{T})$ regret if all the eigenvalues of $X$ are of the same order $\sim d^{-1}$.
    \item We obtain $\mathcal{O}(\sqrt{bT})$ regret if $X=\frac{e^A}{\tr(e^A)}$ for some Hermitian matrix $A$ with $\norm{A}_{\op}\leq b$.
\end{itemize}
Compared to the minimax optimal bound of $\calO\rpar{\sqrt{T\log d}}$, such instance-dependent savings are particularly desirable when $d$ is exponential in certain auxiliary parameters, such as in quantum information (Section~\ref{sec:intro_quantum}). 

Next, for readers familiar with the derivation of MMWU, Eq.~\eqref{eq:regret_ours_summary} should be quite natural. As we elaborate in Section~\ref{sec:prelim_lea_mmwu}, MMWU with an arbitrary \emph{learning rate} $\eta>0$ ensures
\begin{equation}\label{eq:mmwu_eta_prototypical}
\reg_T(X)\lesssim \eta^{-1}S(X||d^{-1}I_d)+\eta T,
\end{equation}
and the optimal $X$-dependent choice of $\eta=\sqrt{S(X||d^{-1}I_d)/T}$ (despite requiring unavailable oracle knowledge) would give us Eq.~\eqref{eq:regret_ours_summary}. From this angle, achieving Eq.~\eqref{eq:regret_ours_summary} can be regarded as a problem of hyperparameter tuning, and one could do this by online model selection (e.g., \cite{foster2017parameter,chen2021impossible}) -- running \emph{Multiplicative Weight Update} (MWU; i.e., the diagonal special case of MMWU) \cite{littlestone1994weighted} over MMWU instances with different $\eta$. So what is the catch?

This is why the second part of Theorem~\ref{thm:lea_main_informal} on the computational complexity is crucial. Online model selection necessarily inflates the computational complexity of the algorithm, making the comparison to MMWU less clear. In contrast, Theorem~\ref{thm:lea_main_informal} improves the regret bound of MMWU while retaining its computational costs, indicating that the advantage over MMWU is ``free''. Such an overarching objective has motivated a fruitful line of works on adaptive (or \emph{parameter-free}) online learning \cite{chaudhuri2009parameter,streeter2012no,luo2015achieving,orabona2016coin,cutkosky2018black,mhammedi2020lipschitz}, and our work fills the void for the matrix LEA problem. 

From a technical perspective, being parameter-free means that instead of Eq.~\eqref{eq:mmwu_eta_prototypical}, we directly achieve the $\mathcal{O}\rpar{\sqrt{T\cdot S(X||d^{-1}I_d)}}$ regret bound without the need to choose any learning rate $\eta$. This relies on the \emph{potential method} which is a general algorithmic framework in online learning (see, e.g., Ref.~\cite{cesa2006prediction}) summarized in Section~\ref{sec:overview_technique}. MMWU corresponds to a special case of this framework with the exponential potential, while improved, parameter-free potential functions have been studied extensively in the vector setting of LEA \cite{mcmahan2014unconstrained,zhang2022pde,harvey2023optimal}. However, their generalization to matrix LEA is not merely a straightforward application, as the noncommutative\footnote{Specifically, the matrices $X_1,\ldots,X_T$ and $G_1,\ldots,G_T$ can have different eigenspaces, therefore do not commute in general. In comparison, the vector LEA problem essentially enforces all these matrices to be diagonal, thus commuting with each other. A generic difficulty in matrix analysis is that plenty of natural properties fail to hold on non-commuting matrices.} nature of the problem introduces significant difficulties for the use of the convexity condition. For context, exploiting convexity through the Jensen's inequality has been a key step for most potential-based online learning algorithms. 

To address this bottleneck, we present an in-depth study on a ``one-sided'' Jensen's trace inequality that the matrix potential method naturally demands (see Section~\ref{sec:overview_technique} and Section~\ref{sec:matrix_Jensen}). We show that intriguingly, such an inequality does not hold for arbitrary convex functions, but it holds on a fairly large function class that includes the iconic parameter-free potentials in the literature. This could be of independent interest, and the proof builds on a novel Laplace transform technique which is the crux of this work. 

Finally, we complement Theorem~\ref{thm:lea_main_informal} with optimality analysis. Our regret bound is optimal following the lower bound in vector LEA \cite{negrea2021minimax}, and we provide a different analysis using elementary techniques in order statistics (see Theorem~\ref{thm:reg_lower}). Computationally, similar to MMWU, the bottleneck of our algorithm is the eigen-decomposition at each time step which has been recently shown to take essentially $\calO\rpar{d^\omega}$ time \cite{banks2023pseudospectral,shah2025hermitian,sobczyk2025deterministic} ($\omega\approx 2.37$ is the matrix multiplication exponent) and $\calO\rpar{d^2}$ memory. We show that $\Omega\rpar{d^2}$ memory is required to achieve $o(T)$ regret, and further generalize this to a combinatorial characterization of the memory lower bound when the comparator $X$ is constrained a priori (see Theorem~\ref{thm:mem_lower}). 

\subsubsection{Applications in online learning of quantum states}\label{sec:intro_quantum}

For applications, we explore the benefits of our algorithm in quantum learning theory. We consider here the online learning of quantum states (see Section~\ref{sec:prelim_quantum} for background and Definition~\ref{def:online_quantum} for the formal definition), where the unknown target quantum state $\rho\in\mathcal{X}$ plays the role of the comparator $X$ in matrix LEA. At each time step $t\in[T]$, the learner outputs a hypothesis quantum state $\rho_t\in\mathcal{X}$, and the adversary then provides a Hermitian observable $O_t$ playing the role of $G_t$ satisfying $\norm{O_t}_\op\leq l$. The learner incurs a convex loss $\ell_t(O_t,\rho)$, and after $T$ steps the final regret is defined as
\begin{align}\label{eq:def_quantum_loss_convex}
\reg_T(\rho)\defeq \sum_{t=1}^T\ell_t(O_t,\rho_t)-\sum_{t=1}^T\ell_t(O_t,\rho).
\end{align}
A common choice of the loss function is the $L_1$ loss $\ell_t(O_t,\rho)=\abs{\tr(O_t\rho_t)-\tr(O_t\rho)}$. In this realizable setting, the regret of the learner is given by
\begin{align}\label{eq:def_quantum_loss}
\reg_T(\rho)=\sum_{t=1}^T\abs{\tr(O_t\rho_t)-\tr(O_t\rho)}.
\end{align}

To apply our algorithm, we first present its extension to \emph{Online Convex Optimization} (OCO; Corollary~\ref{coro:nonlinear_reg}). The regret bound in Theorem~\ref{thm:lea_main_informal} still holds as long as the loss function $\ell_t(O_t,\rho)$'s are Lipschitz with respect to $O_t$. 

Next, we focus on the $L_1$ loss in Eq.~\eqref{eq:def_quantum_loss} which is $l$-Lipschitz. Our regret bound adapts to the quantum relative entropy $S(\rho||d^{-1}I_d)$ of the target state $\rho$ relative to the maximally mixed state $d^{-1}I_d$, therefore intuitively, $\rho$'s that are more ``mixed'' (i.e., with a more evenly distributed spectrum) result in smaller regret bounds. We demonstrate such a benefit in three concrete settings, including learning quantum states corrupted by depolarization noise, random quantum states, and Gibbs states. Finally, we move on to the quadratic loss and show that our algorithm provides improved regret bounds for estimating purity, quantum virtual cooling, and R\'{e}nyi-$2$ correlation function. 

Below is a more detailed summary of these results; also see Table~\ref{tab:quantum}. We also provide the $\epsilon$-mistake bound widely adopted by the quantum literature~\cite{aaronson2018online}, which is the upper bound on the number of iterations with $\abs{\ell_t(O_t,\rho_t)-\ell_t(O_t,\rho)}\geq\epsilon$.

\begin{table}[htbp]
\centering\renewcommand\cellalign{lc}
\resizebox{1.0\columnwidth}{!}{
\begin{tabular}{lllll}
\hline
Class of states & Loss & Guarantee & Regret bound & Mistake bound \\ \hline
General & $L_1$ & Worst & $\mathcal{O}(\sqrt{T\log d})$~\cite{aaronson2018online} & $\mathcal{O}\rpar{\log d/\epsilon^2}$\\
\hline
General & $L_1$ & Worst & \makecell{$\mathcal{O}\rpar{\sqrt{T\cdot S(\rho||d^{-1}I_d)}}$\\ Corollary~\ref{coro:nonlinear_reg}} & $\mathcal{O}\rpar{S(\rho||d^{-1}I_d)/\epsilon^2}$\\
\hline
Noisy circuit $(D,\gamma)$ & $L_1$ & Worst & \makecell{$\mathcal{O}\rpar{(1-\gamma)^D\sqrt{T\log d}}$\\ Corollary~\ref{coro:noisy_state}} & $\mathcal{O}\rpar{(1-\gamma)^{2D}\log d/\epsilon^2}$\\
\hline
\makecell{Subsystem of Haar\\ random state $(d,d')$} & $L_1$ & Average & \makecell{$\mathcal{O}\rpar{\sqrt{Td/d'}}$\\ Corollary~\ref{coro:random_quantum}} & $\mathcal{O}\rpar{d/(d'\epsilon^2)}$\\
\hline
Random product state $(\eta)$ & $L_1$ & Average & \makecell{$\mathcal{O}\rpar{\sqrt{T\log d(1-\eta)}}$\\ Corollary~\ref{coro:random_quantum}} & $\mathcal{O}\rpar{(1-\eta)\log d/\epsilon^2}$\\
\hline
Gibbs state $(\beta=\mathcal{O}(1))$ & $L_1$ & Worst & \makecell{$\mathcal{O}\rpar{\sqrt{\beta T}}$\\ Corollary~\ref{coro:gibbs}} & $\mathcal{O}\rpar{\beta/\epsilon^2}$\\
\hline
Gibbs state (any $\beta$) & $L_1$ & Average & \makecell{$\mathcal{O}\rpar{\sqrt{\beta T}}$\\ Corollary~\ref{coro:gibbs}} & $\mathcal{O}\rpar{\beta/\epsilon^2}$\\
\hline
\makecell{General\\(Quantum virtual cooling)} & $\tr(O_t\rho_t^2)$ & Worst & \makecell{$\mathcal{O}\rpar{\sqrt{T\cdot S(\rho||d^{-1}I_d)}}$\\ Corollary~\ref{coro:nonlinear_quantum}} & $\mathcal{O}\rpar{S(\rho||d^{-1}I_d)/\epsilon^2}$\\
\hline
\makecell{General\\(R\'{e}nyi-$2$ correlation)} & $\tr(O_t\rho_tO_t\rho_t)$ & Worst & \makecell{$\mathcal{O}\rpar{\sqrt{T\cdot S(\rho||d^{-1}I_d)}}$\\ Corollary~\ref{coro:nonlinear_quantum}} & $\mathcal{O}\rpar{S(\rho||d^{-1}I_d)/\epsilon^2}$\\
\hline
\end{tabular}}
\vspace{0.3em}
\caption{A summary of our results for online learning of quantum states. Here, we assume that $\norm{O_t}_{\op}\leq 1$ for simplicity. In addition to the regret bound, we also provide the $\epsilon$-mistake bound widely adopted by the quantum literature~\cite{aaronson2018online}, which is the upper bound on the number of iterations with $\abs{\ell_t(O_t,\rho_t)-\ell_t(O_t,\rho)}\geq\epsilon$.}
\label{tab:quantum}
\end{table}

\vspace{0.5em}\noindent\textbf{Noisy quantum states. }First, we consider quantum states corrupted by noises on near-term quantum devices~\cite{preskill2018quantum}. While such noises erase the useful information which is harmful for computing, they also smooth the spectrum of the target quantum states which simplifies the learning task and can be exploited by our algorithm (but not by MMWU). Specifically, we consider the local depolarization noise, which is a typical noise model for analyzing near-term quantum computation. We show that depolarization noise of rate $\gamma$ can provide an $(1-\gamma)^{1/2}$ multiplicative factor on the regret bound, which can be further boosted to a factor of $(1-\gamma)^D$ if the underlying quantum state is prepared by a quantum circuit of depth $D$ with local depolarization noise at each layer (see Corollary~\ref{coro:noisy_state}).

\vspace{0.5em}\noindent\textbf{Random quantum states. }Another class of quantum states that benefits from our algorithm is random quantum states, which are essential resources in pseudorandomness (quantum cryptography)~\cite{ji2018pseudorandom}, demonstration of quantum advantage in sampling tasks~\cite{arute2019quantum,zhong2020quantum}, and quantum benchmarking~\cite{elben2023randomized}. Here, we show that compared to MMWU, our algorithm can obtain a better regret bound \emph{in the average case} for the online learning of random quantum states. In particular, we consider two scenarios: (i) $\rho$ is a subsystem of a Haar random quantum state~\cite{mele2024introduction}, and (ii) $\rho$ is a random product state and each qubit has a bounded Pauli second moment matrix (see Section~\ref{sec:prelim_quantum} for the definition). For (i), we show that the regret of our algorithm is given by $\mathcal{O}\rpar{\sqrt{Td/d'}}$ when $d\ll d'$, with $d$ and $d'$ being the dimension of the subsystem and the dimension of the full Haar random quantum states. For (ii), we show that our algorithm provides a factor of $\sqrt{1-\eta}$ reduction to the regret of MMWU, where $1-\eta$ is an upper bound on the operator norm of the Pauli second moment matrix (see Corollary~\ref{coro:random_quantum}).

\vspace{0.5em}\noindent\textbf{Gibbs states. }Learning Gibbs states, which are the states of the form $e^{-\beta H}/\tr(e^{-\beta H})$ given Hamiltonian $H$ and inverse temperature $\beta$, also benefit from our algorithm.  The Gibbs state tells us what the equilibrium state of the quantum system will be if it interacts with the environment at a particular temperature and reaches thermal equilibrium. It is widely considered in quantum Gibbs sampling~\cite{chowdhury2017quantum,kastoryano2016quantum,chen2025quantum}, which is the backbone of many quantum algorithms such as semidefinite programming solvers~\cite{brandao2017quantum,van2017quantum,brandao2019quantum,brandao2022faster}, quantum annealing~\cite{montanaro2015quantum}, quantum machine learning~\cite{wiebe2014quantum}, and quantum simulations at finite temperature~\cite{motta2020determining}. We show an $\mathcal{O}\rpar{\beta\sqrt{T}}$ regret bound for predicting properties of Gibbs states at inverse temperature $\beta=\mathcal{O}(1)$ in the worst case and arbitrary $\beta$ in the average case over random Gaussian Hamiltonians or random sparse Hamiltonians in the Pauli basis, improving upon the regret of the standard MMWU algorithm (see Corollary~\ref{coro:gibbs}). 

\vspace{0.5em}\noindent\textbf{Purity, quantum virtual cooling, and R\'{e}nyi-$2$ correlation. }We also prove regret bounds for the online learning of nonlinear quantum properties. Here, we consider the more general loss function in the form of Eq.~\eqref{eq:def_quantum_loss_convex}. To make the function convex with respect to $\rho_t$, we assume $O_t\succeq 0$ is PSD.

The first loss function we consider is of the form $\ell_t(O_t,\rho_t)=\tr(O_t\rho_t^2)$. When $O_t=I$, this task reduces to purity estimation of the given quantum state in an online setting. For a general $O_t$, this task is known as quantum virtual cooling~\cite{cotler2019quantum}. These two quantities play an important role in quantum benchmarking~\cite{eisert2020quantum}, experimental and theoretical quantum (entanglement) entropy (purity) estimation~\cite{islam2015measuring,kaufman2016quantum,brydges2019probing,zhang2021experimental,shaw2024benchmarking,gong2024sample,liu2024exponential}, quantum error mitigation~\cite{cai2023quantum,koczor2021exponential,huggins2021virtual}, quantum principal component analysis~\cite{lloyd2014quantum,huang2020predicting,huang2022quantum,liu2024exponential}, and quantum metrology~\cite{giovannetti2011advances}. In this work, we show that the regret of estimating quantum virtual cooling is of $\mathcal{O}\rpar{\sqrt{T\cdot S(\rho||d^{-1}I_d)}}$ scaling, which is the same as the result with $L_1$ loss (see Corollary~\ref{coro:nonlinear_quantum}).

The second loss function we consider is the R\'{e}nyi-$2$ correlation of the form $\ell_t(O_t,\rho_t)=\tr(O_t\rho_tO_t\rho_t)$. It potentially applies to discovering strong-to-weak spontaneous symmetry breaking (SWSSB) in mixed states~\cite{lessa2025strong}. Here, we again show that the regret of our algorithm is $\mathcal{O}\rpar{\sqrt{T\cdot S(\rho||d^{-1}I_d)}}$ (see Corollary~\ref{coro:nonlinear_quantum}). As the quantum states considered in SWSSB are mixed states, which have more evenly distributed spectra compared to pure states, they can thus benefit from our algorithm.

\subsection{Overview of techniques}\label{sec:overview_technique}

Our results are built on a potential-based framework for matrix online learning. We now sketch its main idea, where the bottleneck is, and how we overcome this bottleneck. 

\vspace{0.5em}\noindent\textbf{Lifting the constraint. }Matrix LEA is a constrained online learning problem on the spectraplex $\Delta_{d\times d}$. As the preparatory step, we first present a rate-preserving reduction (Algorithm~\ref{alg:reduction}) to the unconstrained problem on the space of all Hermitian matrices, $\mathbb{H}_{d\times d}$. This generalizes existing techniques from vector LEA \cite{luo2015achieving,orabona2016coin,cutkosky2018black}, and consequently, we will assume the domain is $\mathbb{H}_{d\times d}$. 

Next, the potential method requires specifying a time-varying \emph{potential function} $\Phi_t:\R\rightarrow\R$ as input. Its argument can be extended to Hermitian matrices in the standard spectral manner: for any $X\in\mathbb{H}_{d\times d}$ with eigen-decomposition $X=\sum_{i=1}^d\lambda_iv_iv_i^\herm$ (where $\lambda_i\in\R$, $v_i\in\C^d$), we define $\Phi_t(X)\defeq \sum_{i=1}^d\Phi_t(\lambda_i)v_iv_i^\herm\in\mathbb{H}_{d\times d}$. 

\vspace{0.5em}\noindent\textbf{Unconstrained algorithm. }Our algorithm then operates on $\mathbb{H}_{d\times d}$ as follows (see Algorithm~\ref{alg:unconstrained}). At the beginning of the $t$-th time step, it computes $S_t\defeq-\sum_{i=1}^{t-1}G_i$ from the observed loss matrices, which serves as a ``sufficient statistic'' of the past (in the sense of Ref.~\cite{foster2018online}). With the knowledge of $\norm{G_t}_\op\leq 1$, the algorithm chooses the unconstrained prediction
\begin{equation*}
X_t=\frac{1}{2}\spar{\Phi_t\rpar{S_t+I}-\Phi_t\rpar{S_t-I}}\in\mathbb{H}_{d\times d}. 
\end{equation*}

\vspace{0.5em}\noindent\textbf{Intuition and choosing $\Phi_t$. }To see the intuition of this algorithm, consider the scalar case ($d=1$). The rationale here is that $X_t$ approximates the derivative $\nabla\Phi_t(S_t)$, which yields $G_tX_t\approx \Phi_t(S_t)-\Phi_t(S_t-G_t)\approx \Phi_{t-1}(S_t)-\Phi_{t}(S_{t+1})$. In fact, with a convex $\Phi_{t}$, the standard scalar Jensen's inequality gives us
\begin{align*}
\Phi_{t}(S_{t+1})&\leq \frac{1-G_t}{2}\Phi_{t}(S_{t}+1)+\frac{1+G_t}{2}\Phi_{t}(S_{t}-1)\tag{$-1\leq G_t\leq 1$}\\
&=\underbrace{\frac{1}{2}\spar{\Phi_{t}(S_{t}+1)+\Phi_{t}(S_{t}-1)}}_{\eqdef\Diamond}-G_tX_t.  
\end{align*}
Furthermore, there are also known choices of $\Phi_t$'s satisfying $\Diamond\leq \Phi_{t-1}(S_t)$ (which is closely related to a discretization of It\^o's formula \cite{harvey2020optimal,zhang2022pde}), with the two iconic ``parameter-free'' examples being
\begin{equation}\label{eq:potential_sketch}
\Phi^\expsq_t(s)=\frac{1}{d\sqrt{t}}\exp\rpar{\frac{s^2}{2t}},\quad\mathrm{and}\quad\Phi^\erfi_t(s)=\frac{\sqrt{t}}{d}\spar{2\int_0^{\frac{s}{\sqrt{2t}}}\rpar{\int_0^u\exp(x^2)\diff x }\diff u-1}.
\end{equation}
Taking a telescopic sum of the inequality leads to a total loss upper bound
\begin{equation*}
\sum_{t=1}^TG_tX_t\leq -\Phi_T\rpar{-\sum_{t=1}^TG_t}+\calO(1),
\end{equation*}
and due the duality between the total loss and the regret \cite{mcmahan2014unconstrained}, we further obtain the corresponding regret bound expressed through the Fenchel conjugate $\Phi^*_T$ of $\Phi_T$, 
\begin{equation*}
\reg_T(X)\leq \Phi_T^*(X)+\calO(1).
\end{equation*}
Specifically, the potential functions in Eq.~\eqref{eq:potential_sketch} have nice Fenchel conjugates such that the regret bound obtained in this way matches the regret bound of default baselines under the optimal oracle tuning -- this is where the name ``parameter-free'' comes from. 

\vspace{0.5em}\noindent\textbf{Challenge of Jensen's inequality. }While the above potential-based analytical strategy has been standard in the scalar and diagonal setting of online learning, the generalization to the noncommutative matrix setting faces a crucial challenge: analogous to the previous scalar Jensen's inequality, we need to show that for the aforementioned convex, parameter-free $\Phi_t$,
\begin{equation}
\tr\spar{\Phi_{t}(S-G)}\leq \tr\spar{\frac{I-G}{2}\Phi_{t}(S+I)+\frac{I+G}{2}\Phi_{t}(S-I)},\quad\forall S,G\in\mathbb{H}_{d\times d}, \norm{G}_\op\leq 1.\label{eq:one_sided_jensen_sketch}
\end{equation}
As we further discuss in Section~\ref{sec:matrix_Jensen}, this deviates from the better-known \emph{Jensen's trace inequality} \cite{hansen2003jensen} as the weighting matrices $\frac{I+G}{2}$ and $\frac{I-G}{2}$ are applied ``only from one side''. Somewhat surprisingly, we are unable to find any existing study in the literature, despite the appeared importance of this one-sided Jensen's trace inequality in matrix linear optimization. 

The highlight of this work is thus a characterization of Eq.~\eqref{eq:one_sided_jensen_sketch}. We first show that unlike the standard Jensen's trace inequality which holds for all convex functions, Eq.~\eqref{eq:one_sided_jensen_sketch} does not hold for arbitrary convex $\Phi_t$, indicating that the one-sided version is substantially different from the known regimes. A specific counterexample is the absolute value function $\Phi_t(s)=\abs{s}$ (Example~\ref{example:absolute}). Nonetheless, a silver lining is that with an exponential function $\Phi_t(s)=\exp(cs);c\in\R$, Eq.~\eqref{eq:one_sided_jensen_sketch} holds due to the celebrated \emph{Golden-Thompson inequality} \cite{golden1965lower,thompson1965inequality}, which suggests thinking about more general $\Phi_t$ as positive linear combinations of exponential functions. Along this line, we prove that

\begin{theorem}[Informal, see Theorem~\ref{thm:main_technical} and Corollary~\ref{thm:phi_itself}]\label{thm:jensen_like_informal}
Eq.~\eqref{eq:one_sided_jensen_sketch} holds if either $\Phi_t$ itself or the second derivative of $\Phi_t$ is the two-sided Laplace transform of a non-negative function.
\end{theorem}

\begin{theorem}[Informal, see Lemma~\ref{lemma:phi_laplace}]\label{thm:potential_transform_informal}
In Eq.~\eqref{eq:potential_sketch}, the ``exp-square potential'' $\Phi^\expsq_t$ is the second derivative of the ``erfi potential'' $\Phi^\erfi_t$, and it is also the two-sided Laplace transform of the positive function $\frac{1}{\sqrt{2\pi}d}\exp\rpar{-\frac{1}{2}tz^2}$ reminiscent of the Gaussian density, such that Theorem~\ref{thm:jensen_like_informal} can be applied. 
\end{theorem}

A high level remark following from Theorem~\ref{thm:potential_transform_informal} is that our algorithm simulates the \emph{averaged} behavior of ``MMWU-like'' base algorithms with Gaussian distributed learning rates, and this intuitively justifies the connection between our objective and the issue of learning rate tuning in MMWU (see Section~\ref{sec:ensemble}). Rigorously, Eq.~\eqref{eq:one_sided_jensen_sketch} and eventually the instance-optimal regret bound of our algorithm follow from the above two results. 

Put into a broader context, the underlying structure of parameter-free potential functions like Eq.~\eqref{eq:potential_sketch} has largely remained mysterious in the literature, and therefore, we believe the above transform domain characterization alongside its connection to the one-sided Jensen's trace inequality are of independent technical interest. For details, Section~\ref{sec:matrix_Jensen} presents our analysis on the one-sided Jensen's trace inequality. Section~\ref{sec:algorithm} specializes it to the parameter-free potentials, and closes the loop with our algorithm and its regret bound. 

\subsection{Related works}

MMWU is a seminal algorithm developed by a number of early works \cite{tsuda2005matrix,kuzmin2007online,warmuth2008randomized}. As matrix LEA is essentially \emph{Online Linear Optimization} (OLO) on the spectraplex, MMWU is the specialization of \emph{Follow the Regularized Leader} (FTRL) with the quantum entropy. See Refs.~\cite{hazan2023introduction,orabona2025modern} for the online learning basics, and Ref.~\cite{tian2025MMWU} for a treatment of MMWU from this angle. 

\vspace{0.5em}\noindent\textbf{Parameter-free online learning. }The adaptivity to the comparator (also known as ``parameter-freeness'') has been studied extensively in various vector online learning settings, including unconstrained OLO on $\R^d$, and the standard vector LEA (i.e., OLO on the probability simplex). There are two different origins for this idea: Ref.~\cite{chaudhuri2009parameter} initiated such study on vector LEA, which was then developed by Refs.~\cite{chernov2010prediction,foster2015adaptive,koolen2015second,luo2015achieving}; in parallel, Ref.~\cite{streeter2012no} studied this idea on general vector OLO, followed by Refs.~\cite{mcmahan2013minimax,mcmahan2014unconstrained}. Later, it was shown that the two regimes can be unified \cite{orabona2016coin,cutkosky2018black,zhang2022optimal}, which reduces the problem to designing better one-dimensional potential functions \cite{mhammedi2020lipschitz,zhang2022pde,cutkosky2024fully,zhang2024improving}. See Ref.~\cite[Section~10]{orabona2025modern} for a summary. 

To date, there are two iconic parameter-free potential functions as shown in Eq.~\eqref{eq:potential_sketch}. The ``exp-square potential'' is due to Ref.~\cite{mcmahan2014unconstrained}, while the quantitatively stronger ``erfi potential'' is due to Ref.~\cite{harvey2020optimal}. We will extend the applicability of both of them to the non-commutative matrix setting. In addition, characterizing them through the Laplace transform is new, although the high level idea was hinted in some sense by an influential early work \cite{koolen2015second}. 

Since the beginning of this field, retaining the computational complexity of the nonadaptive algorithm has been a key consideration. While online model selection can achieve the desirable regret bound in most cases, its practicality has been somewhat questionable; see Ref.~\cite{chaudhuri2009parameter} which initiated the field. Getting rid of that is not only quantitatively stronger, but also cleaner and more coherent with the underlying structure of the problem. 

In terms of comparator-dependent regret lower bounds, Refs.~\cite{streeter2012no,orabona2013dimension,zhang2022pde} presented results for unconstrained OLO on $\R^d$ (also see Ref.~\cite[Section~5]{orabona2025modern}), and Ref.~\cite{negrea2021minimax} presented the only existing result for vector LEA. To our knowledge, the present work is the first to study comparator adaptivity for matrix LEA. En route to our main result, we obtain a comparator adaptive online matrix prediction algorithm on the PSD cone, improving upon the prior work in this setting \cite{foster2017parameter}. 

\vspace{0.5em}\noindent\textbf{Online learning of quantum data. }Online learning of quantum states and quantum processes are widely considered tasks in quantum learning theory. The formal definition of the state-learning problem was given by Ref.~\cite{aaronson2018online}. Using it as the key subroutine, the shadow tomography of quantum states was then studied in a sequence of follow-up works~\cite{aaronson2018shadow,aaronson2019gentle,buadescu2021improved,gong2023learning}. The online learning model for quantum states was studied by Refs.~\cite{yang2020revisiting,chen2020more,lumbreras2022multi,zimmert2022pushing}.  Ref.~\cite{chen2024adaptive} further investigated an adaptive variant of this online learning model where the underlying state may change over time. 

Recently, Refs.~\cite{bansal2025online,raza2024online} considered the task of online learning of quantum processes, and showed improved regret bounds assuming additional prior knowledge. Besides, online learning also serves as a subroutine for learning Pauli channels~\cite{chen2025efficient}.

Finally, we note that our work is part of a larger body of recent results exploring the complexity of learning quantum states. A review of this literature is beyond the scope of this work, and we refer the reader to the survey~\cite{anshu2024survey} for a more thorough overview.

\subsection{Outlook}
In this work, we propose the first algorithm, to the best of our knowledge, achieving an instance-optimal regret bound (with respect to the comparator $X$) for matrix LEA. Our algorithm has the same computational complexity as MMWU, and specifically, its memory complexity is proved to be optimal. Our results are based on the matrix potential method and, crucially, a new ``one-sided'' Jensen's trace inequality, which may be of independent interest. Then, starting from this algorithm backbone, we explore a number of interesting applications in quantum learning theory. For the online learning of quantum states, our algorithm outperforms the best known algorithms in a variety of settings, including learning noisy quantum states, random quantum states, Gibbs states, and predicting non-linear quantum properties. Below, we mention some concrete open questions.

\vspace{0.5em}\noindent\textbf{Precise condition on Jensen's trace inequality. }Regarding the core techniques, our work presents a sufficient condition for the one-sided Jensen's trace inequality Eq.~\eqref{eq:one_sided_jensen_sketch}, and this is general enough for our specific potential functions to apply. Beyond this, as the inequality is supposed to be broadly useful for matrix optimization, an important open question is whether we can obtain a more general, or even sufficient and necessary condition on the considered convex function $\Phi_t$. Specifically, we conjecture that the inequality holds for all monomial functions of even degree, i.e., $\Phi_t(s)=s^{2k}, k\in\N_+$ (Conjecture~\ref{prop:even_poly_conjecture}), towards which we provide some partial result. Another question is whether this inequality can be related or even reduced to classical structures in matrix analysis. 

\vspace{0.5em}\noindent\textbf{Improving the time complexity. }Just like MMWU, our matrix LEA algorithm requires the eigen-decomposition at each iteration, which essentially takes $\calO(d^\omega)$ time in theory and $\calO(d^3)$ time using practical general-purpose methods. As $d$ can be possibly large in practice, it is thus natural to ask if we can bypass this step. A possible solution is low-rank sketching -- this has been shown to improve the time complexity of MMWU \cite {allen2017follow,carmon2019rank}, but its application to our algorithm faces a number of technical challenges due to our use of the somewhat complicated parameter-free potential functions.

\vspace{0.5em}\noindent\textbf{Memory-regret tradeoffs for matrix LEA. }The problem we study directly generalizes the distributional setting of vector LEA, and here we show that our matrix prediction algorithm has the optimal $\calO(d^2)$ memory complexity. However, there is also a non-distributional, single-expert setting of vector LEA, where the learner is only allowed to predict the index of a single expert rather than a distribution over all experts (the direct matrix analogue of this problem would require the prediction $X_t$'s to be rank-one, but the learner is allowed to be randomized). Recently, it has been shown that in this non-distributional vector LEA problem one can achieve a tradeoff between sublinear regret and sublinear memory~\cite{peng2023online,srinivas2022memory,peng2023near}. A natural follow-up question is whether there exist similar memory-regret tradeoffs for the matrix version of this problem. 

\vspace{0.5em}\noindent\textbf{Estimating non-convex quantum properties. }In the quantum part of this work, we extend our matrix LEA algorithm to OCO on $\Delta_{d\times d}$ and explore its quantum applications. Along the way, assumptions are made to ensure such loss functions in the quantum applications are convex, e.g., the observable $O_t$ is assumed to be PSD in Corollary~\ref{coro:nonlinear_quantum}~\cite{huang2022quantum,lessa2025strong}. Going beyond this, it is an important future direction to study how to tackle non-convex, but possibly structured loss functions motivated by quantum applications. An example along this line is the fidelity correlator $\tr\rpar{\sqrt{\sqrt{\rho}O\rho O\sqrt{\rho}}}$, which is another function for detecting SWSSB~\cite{lessa2025strong}. 

\vspace{0.5em}\noindent\textbf{(Quantum) semidefinite programming (SDP). }Finally, one of the most well-known applications of MMWU is solving SDPs \cite{tian2025MMWU}. Therefore, it would be interesting to explore the applications of our algorithm in classical and quantum SDP solvers~\cite{brandao2017quantum,brandao2019quantum,van2017quantum,brandao2022faster}.

\subsection{Roadmap}
For the rest of this paper: Section~\ref{sec:prelim} provides the technical preliminaries. Section~\ref{sec:matrix_Jensen} presents a novel one-sided Jensen's trace inequality, which is the cornerstone of our analysis and the highlight of this work. Section~\ref{sec:algorithm} contains our algorithm and its regret bound, and corresponding lower bounds are presented in Section~\ref{sec:lower_bounds}. Finally, Section~\ref{sec:quantum} presents the application of our algorithm to quantum learning theory. 

\section{Preliminaries}\label{sec:prelim}

In this section, we collect the basic concepts and known results required by this paper. 

\vspace{0.5em}\noindent\textbf{Notations. }$\norm{B}_{\op}$, $\norm{B}_{\tr}$, $\norm{B}_F$, and $\norm{B}_1$ represent the Euclidean operator norm, the trace norm, the Frobenius norm, and the $L_1$ norm of a matrix $B$. $\norm{v}_p$ denotes the $L_p$ norm of a vector $v$. $B^*$ denotes the conjugate transpose of a matrix $B$. $\Delta_d$ denotes the probability simplex, $\Delta_{d\times d}$ denotes the spectraplex, and $\mathbb{H}_{d\times d}$ denotes the Hilbert space of all $d$-dimensional Hermitian matrices. By its standard property, the Frobenius inner product $\inner{A}{B}=\tr(AB)$ between $A,B\in\mathbb{H}_{d\times d}$ is always real. $\log$ denotes the natural logarithm unless noted otherwise. We use $\tilde{\calO}$ and $\tilde{\Theta}$ to hide poly-logarithmic factors in big-O notations $(\mathcal{O},\Omega,\Theta)$. We will also use the small-O notations $(o,\omega)$.

\subsection{Matrix LEA and the MMWU baseline}\label{sec:prelim_lea_mmwu}

From the perspective of online learning, the matrix LEA problem (Definition~\ref{def:matrix_lea}) fits into the general setting of Online Linear Optimization (OLO), where the learner repeatedly makes decisions $x_t$ on a closed and convex set $\calX$ subject to adversarial linear losses $\inner{g_t}{x_t}$. Consequently, it is well-known that MMWU can be viewed as an instance of Follow the Regularized Leader (FTRL), a celebrated algorithmic template tackling generic OLO problems. We now cover the basics of this connection which leads to the standard $\calO\rpar{\sqrt{T\log d}}$ regret bound of MMWU. A more comprehensive treatment can be found in Ref.~\cite{tian2025MMWU}.

In general, each instance of FTRL is specified by a strictly convex, closed and proper function $\psi:\calX\rightarrow\R$ called a regularizer, as well as a learning rate $\eta_t\in\R_+$ which is non-increasing with respect to $t$. Given these, the $t$-th decision of the algorithm, denoted by a generic $x_t\in\calX$, is defined as the minimizer of the regularized cumulative loss, i.e., 
\begin{equation*}
x_t=\argmin_{x\in\calX}\spar{\eta_t^{-1}\psi(x)+\sum_{i=1}^{t-1}\inner{g_i}{x}}. 
\end{equation*}
With $\psi^*$ representing the convex conjugate of the regularizer $\psi$ and $\nabla\psi^*$ representing its gradient, one could use standard convex duality \cite[Theorem~6.16]{orabona2025modern} to rewrite the update as $x_t=\nabla\psi^*\rpar{-\eta_t\sum_{i=1}^{t-1}g_i}$. Then, if the regularizer $\psi$ is $\mu$-strongly-convex with respect to a norm $\norm{\cdot}$, the algorithm guarantees the regret bound \cite[Corollary~7.7]{orabona2025modern}
\begin{equation}\label{eq:generic_ftrl_bound}
\sum_{t=1}^T\inner{g_t}{x_t-u}\leq \frac{\psi(u)-\min_{u'\in\calX}\psi(u')}{\eta_T}+\frac{1}{2\mu}\sum_{t=1}^T\eta_t\norm{g_t}_*^2,\quad\forall u\in\calX,
\end{equation}
where $\norm{\cdot}_*$ is the dual norm of $\norm{\cdot}$. 

Specializing the domain $\calX$ to the spectraplex $\Delta_{d\times d}$, MMWU picks $\psi$ as the negative \emph{quantum entropy} (a.k.a., von Neumann entropy), $\psi(X)=S(X)\defeq\inner{X}{\log X}$. The corresponding $\nabla\psi^*$ is the normalized matrix exponential, which yields the MMWU update
\begin{equation}\label{eq:mmwu_update}
X_t=\frac{\exp\rpar{-\eta_t\sum_{i=1}^{t-1}G_i}}{\tr\exp\rpar{-\eta_t\sum_{i=1}^{t-1}G_i}}.
\end{equation}
Specifically, such a $\psi$ is $1$-strongly convex with respect to $\norm{\cdot}_\tr$ \cite[Corollary~1]{tian2025MMWU}, whose dual norm is $\norm{\cdot}_\op$. Furthermore, if we write $S(\cdot||\cdot)$ as the \emph{Bregman divergence} induced by this $\psi$ (a.k.a. \emph{quantum relative entropy}), then it can be shown that $\psi(X)-\min_{X'\in\calX}\psi(X')=S(X||d^{-1}I_d)\leq \log d$. Combining these with Eq.~\eqref{eq:generic_ftrl_bound}, we obtain the standard MMWU regret bound
\begin{equation}\label{eq:mmwu_regret_bound}
\reg_T(X)\leq \frac{S(X||d^{-1}I_d)}{\eta_T}+\frac{1}{2}\sum_{t=1}^T\eta_t\norm{G_t}_\op^2,\quad \forall X\in\Delta_{d\times d}.
\end{equation}
If $\norm{G_t}_\op\leq 1$ for all $t$, then with $\eta_t=\sqrt{t^{-1}\log d}$ the RHS becomes $\mathcal{O}\rpar{\sqrt{T\log d}}$, matching the lower bound in the diagonal setting (i.e., vector LEA) \cite{cesa1997use}. 

\vspace{0.5em}\noindent\textbf{Adaptivity to $X$. }An inspection of Eq.~\eqref{eq:mmwu_regret_bound} suggests a natural way to do better: suppose $S(X||d^{-1}I_d)$ is known beforehand, then one may pick $\eta_t=\sqrt{t^{-1}S(X||X_0)}$ to obtain the improved $\mathcal{O}\rpar{\sqrt{T\cdot S(X||d^{-1}I_d)}}$ regret which adapts to the complexity of each comparator $X$. Even without knowing $S(X||d^{-1}I_d)$, it has been fairly standard to achieve this by running a vector LEA algorithm on top of multiple MMWU instances with different $\eta_t$ values, which selects the best one on the fly; see Refs.~\cite{foster2017parameter,chen2021impossible}. The problem is that doing so inflates the computational complexity of the algorithm, therefore the field of parameter-free online learning has mostly revolved around achieving adaptive regret bounds without LEA-based online model selection \cite[Section~9]{orabona2025modern}. 

\subsection{Matrix inequalities}

Next, we introduce several known matrix inequalities relevant to this work. All of them become trivial statements when the involved matrices commute, so the point here is that the general noncommutative setting requires special care. 

The first is the celebrated Golden-Thompson inequality \cite{golden1965lower,thompson1965inequality}. For two commuting matrices $A$ and $B$, we have the intuitive matrix equality $\exp(AB)=\exp(A)\exp(B)$. The Golden-Thompson inequality characterized the situation when $A$ and $B$ do not commute but are Hermitian. 

\begin{lemma}[Golden-Thompson inequality]
\label{lemma:golden_thompson}
For any Hermitian matrices $A$ and $B$, 
\begin{equation*}
    \tr\spar{\exp\rpar{A+B}} \le \tr\spar{\exp{A}\exp{B}}.
\end{equation*}
\end{lemma}

The following is a version of von Neumann's trace inequality \cite{mirsky1975trace}, adapted from Ref.~\cite{tian2025MatrixAnalysis}.

\begin{lemma}[von Neumann's trace inequality]\label{lemma:von_neumann}
Let Hermitian matrices $A,B\in\mathbb{H}_{d\times d}$ have eigen-decompositions $U\Lambda U^*$ and $V\Lambda'V^*$ respectively, where $\Lambda$ and $\Lambda'$ are diagonal matrices with non-decreasing entries $\lambda_1\leq\ldots\leq\lambda_d$ and $\lambda'_1\leq\ldots\leq\lambda'_d$. Then, 
\begin{equation*}
\tr\spar{AB} \leq \sum_{i=1}^d\lambda_i\lambda'_i,
\end{equation*}
where the equality holds if $U=V$. 
\end{lemma}

A recurring theme of matrix analysis is that the trace of a interleaving matrix product (between two Hermitian matrices) can sometimes be bounded by the trace of a ``disentangled'' matrix product; see Ref.~\cite{tian2025MatrixAnalysis}. The following is the simplest realization of this idea. 

\begin{lemma}[Disentangling matrix product]\label{lemma:disentangle}
For any Hermitian matrices $A$ and $B$, 
\begin{equation*}
\tr\spar{ABAB}\leq \tr\spar{A^2B^2}.
\end{equation*}
\end{lemma}

Next, we present the standard ``two-sided'' Jensen's trace inequality \cite{hansen2003jensen}, which will be contrasted with our ``one-sided'' version in Section~\ref{sec:matrix_Jensen}.

\begin{lemma}[Jensen's trace inequality]
\label{lemma:jensen_trace}
For any convex function $\Phi:\R\rightarrow\R$, Hermitian matrices $X_i$, and any matrices $A_i$ such that $\sum_{i=1}^k A_i^*A_i=I$, the following inequality holds:
\begin{align*}
\tr\spar{\Phi\rpar{\sum_{i=1}^k A_i^*X_i A_i}} \le \tr\spar{\sum_{i=1}^k A_i^* \Phi\rpar{X_i} A_i}.
\end{align*}
\end{lemma}
The intuition is that the Hermitian matrices $A_1^*A_1,\ldots,A_k^*A_k$ are analogous to the ``weights'' of a convex combination. However, each of these weights is split into two matrices $A_i^*$ and $A_i$ placed on both sides of the matrix variable $X_i$, forming a ``two-sided'' sandwich structure. 

Related is the Jensen's operator inequality \cite{hansen2003jensen}. Applied to Hermitian matrices, it differs from Lemma~\ref{lemma:jensen_trace} in that (i) the function $\Phi$ is required to be operator convex rather than just scalar convex, and (ii) the trace function on both sides is removed thus the inequality is in the stronger, Loewner order. 

\subsection{Basic results in quantum information}\label{sec:prelim_quantum}
Finally, we recap some standard definitions and calculations in quantum information~\cite{nielsen2010quantum}. 

A general $n$-qubit quantum state can be represented as a $d=2^n$ dimensional Hermitian PSD matrix $\rho\in\Delta_{d\times d}$ with unit trace $\tr(\rho)=1$. When the state is rank-$1$, it is a \emph{pure state} and usually denoted as the state vector $\ket{\psi}$ or $\ket{\phi}$ throughout this paper. We denote the all-zero pure state as $\ket{0}$. The von Neumann entropy of $\rho$ is given by $S(\rho)=\tr(\rho\log\rho)$.

A quantum \emph{observable}, or quantum operator, is a $d$-dimensional Hermitian matrix $O\in\mathbb{H}_{d\times d}$. Given a quantum state $\rho$ which can be seen as a distribution over pure states, the expectation of $O$ with respect to $\rho$ is given by $\tr(O\rho)$. 

\vspace{0.5em}\noindent\textbf{Problem setting. }Given the above, we define the problem of online learning of quantum states~\cite{aaronson2018online}, considered in Section~\ref{sec:quantum}.
\begin{definition}[Online learning of quantum states]\label{def:online_quantum}
An algorithm is initiated with a starting quantum state $\rho_1$, usually the maximal mixed state $\rho_1=I_d/d$, and makes predictions for $T$ rounds. There is an underlying unknown state $\rho$ referred to as the ground truth or target to be learned. At the $t$-th time step (for $t\in[T]$),
\begin{itemize}
    \item The algorithm commits a quantum state $\rho_t\in\mathcal{X}$ based on the previous history.
    \item The adversary reveals an observable $O_t$ after seeing $\rho_t$, where $\norm{O_t}_\op\leq l$.
    \item In the standard case, the algorithm suffers the $L_1$ loss $\ell_t(O_t,\rho_t)=\abs{\tr(O_t\rho_t)-\tr(O_t\rho)}$. More generally, the algorithm suffers the loss $\ell_t(O_t,\rho_t)$.
\end{itemize}
The regret of the algorithm is defined as 
\[
\reg_T(\rho)\defeq \sum_{t=1}^T\ell_t(O_t,\rho_t)-\sum_{t=1}^T\ell_t(O_t,\rho).
\]
\end{definition}

Besides the basic definitions, several specific concepts are required. 

\vspace{0.5em}\noindent\textbf{Subsystem. }Given a quantum state $\rho\in\Delta_{d'}$, the $d$-dimensional \emph{subsystem} of $\rho$ arises when the total space factorizes as
\begin{align*}
\calH_{d'}\cong\calH_d\otimes\calH_{d'/d}
\end{align*}
so that the full system can be viewed as composed of two parts: one of dimension $d$ and the other of dimension $d'/d$. To describe just the $d$-dimensional subsystem, you take the partial trace of 
$\rho$ over the complementary factor. The resulting reduced density matrix is again a valid density matrix in $\Delta_{d\times d}$. We refer to Ref.~\cite{nielsen2010quantum} for more mathematical details of partial trace operations.

\vspace{0.5em}\noindent\textbf{Haar random state. }In Section~\ref{sec:quantum_noisy_random}, we consider random pure states based on \emph{Haar random unitaries}. The Haar measure $\mu$ on the unitary group $U(d)$ is the unique probability measure that is invariant under left- and right-multiplication, i.e.
\[
\mathbb{E}_{U\sim\mu}f(UV)=\mathbb{E}_{U\sim\mu}f(VU)=\mathbb{E}_{U\sim\mu}f(U)
\]
for all $V\in U(d)$. We can also define a unique rotation invariant measure on states by $U\ket{\psi}$ for $U\sim\mu$.  Here, $\ket{\psi}$ can be any pure quantum state. Slightly abusing the notation, we will write $\psi\sim\mu$. Given a Haar random state $\ket{\psi}$ of $d'$ dimensions, the Page formula~\cite{page1993average} shows that the average von Neumann entropy of a $d$-dimensional subsystem is of the scaling $\sim \log d - d/d'$ when $d\ll d'$. 

\vspace{0.5em}\noindent\textbf{Pauli second moment matrix. }We introduce single-qubit Pauli matrices:
\begin{align}
I=\begin{pmatrix}1 & 0\\0& 1\end{pmatrix}\,, \qquad X=\begin{pmatrix}0 & 1\\1& 0\end{pmatrix}\,, \qquad Y=\begin{pmatrix}0 & -i\\i& 0\end{pmatrix}\,, \qquad Z=\begin{pmatrix}1 & 0\\0& -1\end{pmatrix}.
\end{align}
We also denote them as $I,P_1,P_2,P_3$. Given a random single-qubit state $\sigma\sim\mathcal{S}$ for ensemble $\mathcal{S}$, the Pauli second moment matrix $S$ is defined as
\begin{align*}
S_{i,j}=\mathbb{E}_{\sigma\sim\mathcal{S}}[\tr(P_i\sigma)\tr(P_j\sigma)],\quad i,j=1,2,3.
\end{align*}
As a random single-qubit state can also be written as a vector within a Bloch sphere as $\sigma=\frac12(I+r\cdot\vec{\sigma})\sim\mathcal{S}$, we have $S=\mathbb{E}_{r\sim\mathcal{S}}[rr^\top]$.

\vspace{0.5em}\noindent\textbf{Pauli basis. }We denote $\cP_n=\{I,X,Y,Z\}^{\otimes n}$ the set of $n$-qubit Pauli observables. Any Hermitian matrix $O$ can be decomposed into a linear combination of $n$-qubit Pauli observables with real coefficients. We thus also called $\cP_n$ the Pauli basis. 

\vspace{0.5em}\noindent\textbf{Gibbs state. }In Section~\ref{sec:quantum_gibbs}, we consider Gibbs states. The dynamics of the quantum system can be described by a Hermitian Hamiltonian $H$. Given an inverse temperature $\beta$, the Gibbs state of a Hamiltonian is given by $e^{-\beta H}/\tr(e^{-\beta H})$, which is the equilibrium state at this temperature when the quantum system evolves under the Hamiltonian $H$. 

\section{One-Sided Jensen's Trace Inequality}\label{sec:matrix_Jensen}

This section studies the main technical component of this paper. Let us consider Hermitian matrices $S$ and $G$ satisfying $\norm{G}_\op\leq \eps$. We want to find conditions on the convex function $\Phi:\R\rightarrow\R$ such that
\begin{align}
    \label{eq:def-jensen}
    \tr[\Phi(S+G)] \leq \tr\spar{\frac{\eps I+G}{2\eps}\Phi(S+\eps I)+\frac{\eps I-G}{2\eps}\Phi(S-\eps I)},\quad\forall S,G\in\mathbb{H}_{d\times d}, \norm{G}_\op\leq \eps.
\end{align}
As we show later in Section~\ref{sec:algorithm}, Eq.~\eqref{eq:def-jensen} naturally arises in the matrix generalization of the potential method, therefore analyzing it will lead to concrete and important algorithmic benefits. However, despite our best effort, we are unable to find any existing study of this inequality in the literature. 

To proceed, we will call Eq.~\eqref{eq:def-jensen} a \emph{one-sided Jensen's trace inequality}. The terminology can be justified as follows. 

\begin{itemize}
\item In the one-dimensional setting (where $S$ and $G$ are scalars), Eq.~\eqref{eq:def-jensen} holds for all convex functions $\Phi$ due to the scalar Jensen's inequality. Here, notice that regardless of dimensionality we always have
\begin{equation*}
S+G=\frac{\eps I+G}{2\eps}(S+\eps I)+\frac{\eps I-G}{2\eps}(S-\eps I).
\end{equation*}
Similarly, the inequality holds for all convex $\Phi$ when $S$ and $G$ commute, by combining the one-dimensional Jensen's inequality with eigen-decomposition. 

\item The term ``one-sided'' is used to contrast Eq.~\eqref{eq:def-jensen} with the existing ``two-sided'' Jensen's trace inequality (the $k=2$ special case of Lemma~\ref{lemma:jensen_trace}), where each ``weighting matrix'' $A_i^*A_i$ is expressed as a product, and the matrix variable $X_i$ is sandwiched by $A_i^*$ and $A_i$ from two sides. There is another difference: in Lemma~\ref{lemma:jensen_trace} (with $k=2$) the matrix variables $X_1$ and $X_2$ do not necessarily commute, while their counterparts $S+\eps I$ and $S-\eps I$ in Eq.~\eqref{eq:def-jensen} do, meaning that our problem has an additional structure. 
\end{itemize}

It appears to us that Eq.~\eqref{eq:def-jensen} is intriguingly different from the known regimes of matrix trace inequalities. To see this, let us first consider applying Lemma~\ref{lemma:jensen_trace} to prove Eq.~\eqref{eq:def-jensen}. By definition, the weighting matrices $\frac{\eps I+G}{2\eps}$ and $\frac{\eps I-G}{2\eps}$ in Eq.~\eqref{eq:def-jensen} are both Hermitian and PSD, therefore their square roots are well-defined and also Hermitian. Then, starting from the RHS of Eq.~\eqref{eq:def-jensen}, we have for all convex $\Phi$,
\begin{align*}
\rhs&=\tr\spar{\sqrt{\frac{\eps I+G}{2\eps}}\Phi(S+\eps I)\sqrt{\frac{\eps I+G}{2\eps}}+\sqrt{\frac{\eps I-G}{2\eps}}\Phi(S-\eps I)\sqrt{\frac{\eps I-G}{2\eps}}}\tag{cyclic property of trace}\\
&\geq \tr\spar{\Phi\rpar{\sqrt{\frac{\eps I+G}{2\eps}}(S+\eps I)\sqrt{\frac{\eps I+G}{2\eps}}+\sqrt{\frac{\eps I-G}{2\eps}}(S-\eps I)\sqrt{\frac{\eps I-G}{2\eps}}}}\tag{Lemma~\ref{lemma:jensen_trace}},\\
&=\tr\spar{\Phi\rpar{\sqrt{\frac{\eps I+G}{2\eps}}S\sqrt{\frac{\eps I+G}{2\eps}}+\sqrt{\frac{\eps I-G}{2\eps}}S\sqrt{\frac{\eps I-G}{2\eps}}+G}},
\end{align*}
but the obtained expression is not necessarily larger than our target $\tr[\Phi(S+G)]$. 

In fact, we can even construct a counterexample showing that the convexity of $\Phi$ alone is not enough for Eq.~\eqref{eq:def-jensen}, meaning that such a proof strategy does not work. 

\begin{example}[Absolute value]\label{example:absolute}
Let $\Phi(x)=\abs{x}$, $\eps=1$ and consider $2\times 2$ real symmetric matrices $S = \begin{pmatrix} 0 & 1 \\ 1 & 0 \end{pmatrix}$ and $G = \begin{pmatrix} 1 & 0 \\ 0 & -1 \end{pmatrix}$. Then, as shown in Appendix~\ref{app:counterexample}, Eq.~\eqref{eq:def-jensen} does not hold. 
\end{example}

On the bright side, there are also simple cases of $\Phi$ for which Eq.~\eqref{eq:def-jensen} can be verified by ``brute force'', including $\Phi(x)=1, x, x^2, x^4$. 

\begin{example}[Affine function and quadratic function]\label{example:affine}
Eq.~\eqref{eq:def-jensen} holds with equality if $\Phi(x)=ax+b$ for arbitrary $a,b\in\R$. Besides, Eq.~\eqref{eq:def-jensen} holds if $\Phi(x)=x^2$. 
\end{example}

The derivation for the case of $x^4$ provides a motivating hint to the source of the difficulty. 

\begin{example}[Monomial of degree $4$]\label{example:x_four}
Eq.~\eqref{eq:def-jensen} holds if $\Phi(x)=x^4$. 

To show this, we expand both sides of Eq.~\eqref{eq:def-jensen}. Using the cyclic property of trace, 
\begin{align*}
\tr\spar{(S+G)^{4}}&=\tr\spar{S^4+4S^3G+4S^2G^2+2GSGS+4SG^3+G^4}\\
&\leq \tr\spar{S^4+4S^3G+6S^2G^2+4SG^3+G^4}.\tag{Lemma~\ref{lemma:disentangle}}
\end{align*}
On the other side, 
\begin{align*}
(S+\eps I)^4=S^4+4\eps S^3+6\eps^2 S^2+4\eps ^3S+\eps^4I,\\
(S-\eps I)^4=S^4-4\eps S^3+6\eps^2 S^2-4\eps^3 S+\eps^4 I,
\end{align*}
therefore
\begin{equation*}
\tr\spar{\frac{\eps I+G}{2\eps}(S+\eps I)^4+\frac{\eps I-G}{2\eps}(S-\eps I)^4}= \tr\spar{S^4+4S^3G+6\eps^2S^2+4\eps^2 SG+\eps^4I}. 
\end{equation*}

Comparing the above, to prove Eq.~\eqref{eq:def-jensen} it suffices to show that
\begin{equation}\label{eq:example_4}
\tr\spar{6S^2G^2+4SG^3+G^4}\leq \tr\spar{6\eps^2S^2+4\eps^2 SG+\eps^4I}.
\end{equation}
To this end, notice that
\begin{equation*}
\lhs=\tr\spar{\rpar{6S^2+2SG+2GS+G^2}G^2}=\tr\spar{\rpar{(2S+G)^2+2S^2}G^2}.
\end{equation*}
Since $(2S+G)^2+2S^2$ is PSD, and $G^2\preceq \eps^2I$, we have
\begin{align*}
\tr\spar{\rpar{(2S+G)^2+2S^2}G^2}&\leq \eps^2\tr\spar{(2S+G)^2+2S^2}\\
&=\tr\spar{6\eps^2S^2+4\eps^2SG+\eps^2G^2}\\
&\leq \tr\spar{6\eps^2S^2+4\eps^2SG+\eps^4I},
\end{align*}
which verifies Eq.~\eqref{eq:example_4}.
\end{example}

The above derivation has two main ingredients. First, the trace of the interleaving matrix product $\tr[GSGS]$ is bounded by the trace of the disentangled product $\tr[S^2G^2]$. Although such a special case is straightforward due to Lemma~\ref{lemma:disentangle}, handling general interleaving products $\prod_{i=1}^{2k}X_i, X_i\in\{G,S\}$ is difficult, cf., Ref.~\cite{lee2021golden}. Second, after disentangling all interleaving matrix products, we formulate the obtained polynomial as a PSD matrix multiplying $G^k$ for some $k$, such that $\norm{G}_\op\leq\eps$ can be applied to reduce the degree of the polynomial (with respect to $G$) to one. This is also supposed to be challenging in general. Therefore, although we tend to believe that Eq.~\eqref{eq:def-jensen} holds for all monomial functions of even degree, i.e., $\Phi(x)=x^{2k}, k\in\N_+$ (and thus a function class generated by Taylor series), we did not successfully prove it (see Section~\ref{subsec:alternative_perspective}). 

The point of this discussion is to justify the nontrivial nature of Eq.~\eqref{eq:def-jensen} and motivate our solution introduced next. 

\subsection{Main result}

Our characterization of Eq.~\eqref{eq:def-jensen} is a major deviation from the strategy above. First, an important observation is that Eq.~\eqref{eq:def-jensen} holds for all exponential functions, essentially due to the Golden-Thompson inequality (Lemma~\ref{lemma:golden_thompson}). 

\begin{lemma}[Exponential function]\label{lemma:exp-case}
For any $c\in\R$, Eq.~\eqref{eq:def-jensen} holds if $\Phi(x)=\exp(cx)$.
\end{lemma}

\begin{proof}[Proof of Lemma~\ref{lemma:exp-case}]
Consider the function $f(\lambda)=\frac{\eps +\lambda}{2\eps }\exp(2c\eps )+\frac{\eps -\lambda}{2\eps }-\exp(c(\eps +\lambda))$ which is concave with respect to $\lambda$. Regardless of $c$ and $\eps $ we always have $f(\eps )=f(-\eps )=0$, therefore $f(\lambda)\geq 0$ for all $\lambda\in[-\eps ,\eps ]$. Then, the lemma follows from
\begin{align*}
\tr[\exp\rpar{c(S+G)}]&\leq \tr[\exp\rpar{c(\eps I+G)}\exp\rpar{c(S-\eps I)}]\tag{Lemma~\ref{lemma:golden_thompson}}\\
&\leq \tr\spar{\rpar{\frac{\eps I+G}{2\eps }\exp(2c\eps )+\frac{\eps I-G}{2\eps }}\exp\rpar{c(S-\eps I)}}\tag{$f(\lambda)\geq 0$ for all $\lambda\in[-\eps ,\eps ]$, and $\exp\rpar{c(S-\eps I)}$ is PSD}\\
&=\tr\spar{\frac{\eps I+G}{2\eps }\exp(2c\eps I)\exp\rpar{c(S-\eps I)}+\frac{\eps I-G}{2\eps }\exp\rpar{c(S-\eps I)}}\\
&=\tr\spar{\frac{\eps I+G}{2\eps }\exp\rpar{c(S+\eps I)}+\frac{\eps I-G}{2\eps }\exp\rpar{c(S-\eps I)}}.\qedhere
\end{align*}
\end{proof}

Why is the Golden-Thompson inequality relevant here? Let us consider its proof strategy \cite{lee2021golden,tian2025MatrixAnalysis}: due to the Lie-Trotter formula, the matrix exponential $\exp(A+B)$ can be approximated by a matrix interleaving product, and the latter can be disentangled while increasing the trace (similar to Lemma~\ref{lemma:disentangle}). Therefore, the proof of Lemma~\ref{lemma:exp-case} is essentially in the same vein as Example~\ref{example:x_four}, and we bypass the aforementioned difficulty by exploiting the special structure of the exponential function. 

Another comment is that the exponential function is not operator convex \cite[Section~8.4.5]{tropp2015introduction}, meaning that Eq.~\eqref{eq:def-jensen} does not reduce to the Jensen's operator inequality. 

Next, notice that Eq.~\eqref{eq:def-jensen} is linear in $\Phi$ in the sense that given any two functions $\Phi_1, \Phi_2$ satisfying Eq.~\eqref{eq:def-jensen}, any linear combination of them with non-negative coefficients also satisfies Eq.~\eqref{eq:def-jensen}. This motivates us to consider the $\Phi$'s that are non-negative linear combinations of the exponential function, or in other words, such $\Phi$'s are the Laplace transform of non-negative functions. The following theorem extends this idea to a characterization from the second derivative of $\Phi$. Here, the condition on $\Phi$ ensures its convexity, therefore we do not need to assume it separately. 

\begin{theorem}[Main technical result]\label{thm:main_technical}
Suppose that for some non-negative function $\mu:\R\rightarrow\R$, the second derivative of a function $\Phi:\R\rightarrow\R$ can be written as
\begin{equation*}
\Phi''(x) = \int_{-\infty}^{\infty} \mu(t)\exp(-tx) \diff t.
\end{equation*}
Then, Eq.~\eqref{eq:def-jensen} holds, i.e., for all Hermitian matrices $S$ and $G$ satisfying $\norm{G}_\op\leq \eps$, we have
\begin{equation*}
\tr[\Phi(S+G)]\leq \tr\spar{\frac{\eps I+G}{2\eps}\Phi(S+\eps I)+\frac{\eps I-G}{2\eps}\Phi(S-\eps I)}.
\end{equation*}
\end{theorem}

\begin{proof}[Proof of Theorem~\ref{thm:main_technical}]
The existence of $\mu$'s Laplace transform implies $\mu$ is measurable. And since $\mu$ is also non-negative, Fubini's theorem can be applied to switch the order of the following integration. By integrating twice,
\begin{align*}
\Phi(x) &= \Phi(0)+\Phi'(0)x+ \int_{0}^x \left[\int_{0}^{y} \left[\int_{-\infty}^{\infty} \mu(t)\exp(-tz) \diff t \right]\diff z \right]\diff y\\
&=\Phi(0)+\Phi'(0)x+ \int_{-\infty}^{\infty} \mu(t)\left[\int_{0}^x \left[ \int_{0}^{y} \exp(-tz) \diff z \right]\diff y\right]\diff t \tag{Fubini's theorem}\\
&=\Phi(0)+\Phi'(0)x+\int_{-\infty}^{\infty} \mu(t)\cdot t^{-2}\spar{\exp(-tx)-1+tx}\diff t.
\end{align*}
As the targeted inequality holds for the affine function $\Phi_{\mathrm{aff}}(x) = \Phi(0)x+\Phi'(0)x$ with equality (Example~\ref{example:affine}), we can drop them and only consider the last integral term. 

Next, define the kernel function $\phi_t(x) = t^{-2}\spar{\exp(-tx)-1+tx}$ with an external parameter $t\in\R$; when $t=0$, $\phi_t(x)=\frac{1}{2}x^2$ by the limit. The integral left above is $\int_{-\infty}^\infty\mu(t)\phi_t(x)\diff t$, and since $\mu(t)$ is non-negative, it suffices to show that with an arbitrary $t\in\R$, 
\begin{equation} \label{eq:kernel_ineq}
\tr[\phi_t(S+G)] \leq \tr\spar{\frac{\eps I+G}{2\eps}\phi_t(S+\eps I)+\frac{\eps I-G}{2\eps}\phi_t(S-\eps I)},\quad \forall S,G\in\mathbb{H}_{d\times d}, \norm{G}_\op\leq \eps.
\end{equation}

If $t=0$, this holds due to Example~\ref{example:affine}. If $t\neq 0$, we can drop the affine parts of $\Phi_t$ as before, and further drop the $t^{-2}$ multiplicative factor since it is positive. This reduces the target to showing
\begin{equation*}
\tr[\exp(-t(S+G))] \le \tr\spar{\frac{\eps I+G}{2\eps}\exp\rpar{-t(S+\eps I)}+\frac{\eps I-G}{2\eps}\exp\rpar{-t(S-\eps I)}}.
\end{equation*}
which follows from Lemma~\ref{lemma:exp-case}.
\end{proof}

A direct corollary is a characterization of Eq.~\eqref{eq:def-jensen} from the function $\Phi$ itself, rather than its second derivative. The proof uses the analytic property of Laplace transform. 

\begin{corollary}\label{thm:phi_itself}
Eq.~\eqref{eq:def-jensen} holds if for some non-negative function $\mu:\R\rightarrow\R$, $\Phi$ itself can be written as
\begin{equation*}
\Phi(x) = \int_{-\infty}^{\infty} \mu(t)\exp(-tx) \diff t.
\end{equation*}
\end{corollary}

\begin{proof}[Proof of Corollary~\ref{thm:phi_itself}]
The function $\Phi$ defined as a Laplace transform is infinitely differentiable. Its second derivative can be found by differentiating under the integral sign:
\[
\Phi''(x) = \frac{\diff ^2}{\diff x^2} \left( \int_{-\infty}^{\infty} \mu(t)\exp(-tx) \diff t \right) = \int_{-\infty}^{\infty} t^2\mu(t)\exp(-tx) \diff t.
\]
By defining a new non-negative function $\nu(t) = t^2\mu(t)$, we see that $\Phi''(x)$ takes the form required by Theorem~\ref{thm:main_technical}.
\end{proof}

Below are some discussion on these results. 

\vspace{0.5em}\noindent\textbf{Relation between Theorem~\ref{thm:main_technical} and Corollary~\ref{thm:phi_itself}. }First, while Corollary~\ref{thm:phi_itself} can be directly proved from Lemma~\ref{lemma:exp-case}, it is not as general as Theorem~\ref{thm:main_technical}. To see this, recall from its proof that Theorem~\ref{thm:main_technical} handles functions of the form
\begin{equation*}
\Phi(x) = C_1 x + C_2 + \int_{-\infty}^{\infty} \mu(t)\phi_t(x)\diff t,
\end{equation*}
where the kernel $\phi_t(x) = t^{-2}\spar{\exp(-tx)-1+tx}$ is regular at $t=0$, meaning that the integral can converge even if $\mu(t)$ does not vanish around $t=0$. Now suppose we want to prove Theorem~\ref{thm:main_technical} using Corollary~\ref{thm:phi_itself}, then the natural strategy is splitting the integral in the above $\Phi$ into
\begin{equation*}
\int_{-\infty}^{\infty} \frac{\mu(t)}{t^2}\exp(-tx)\diff t - \int_{-\infty}^{\infty} \frac{\mu(t)}{t^2}\diff t + x\int_{-\infty}^{\infty} \frac{\mu(t)}{t}\diff t,
\end{equation*}
where the first term fits into the assumption of Corollary~\ref{thm:phi_itself}, and the rest is affine in $x$ thus does not matter. However, this splitting step is only valid if each obtained integral converges independently, which is violated by $\mu$'s that do not vanish around $0$.

\vspace{0.5em}\noindent\textbf{Derivatives of other orders. }Second, we discuss why we only place the Laplace transform assumption on the second derivative of $\Phi$, rather than derivatives of other orders. Generalizing the strategy of Theorem~\ref{thm:main_technical}, if the $n$-th derivative of $\Phi$, denoted by $\Phi^{(n)}$, is the Laplace transform of a non-negative function $\mu$, then by Taylor's theorem with integral remainder,
\begin{align*}
\Phi(x)&=\sum_{k=0}^{n-1}\frac{\Phi^{(k)}(0)}{k!}x^k+\frac{1}{(n-1)!}\int_0^x\Phi^{(n)}(z)(x-z)^{n-1}\diff z\\
&=\sum_{k=0}^{n-1}\frac{\Phi^{(k)}(0)}{k!}x^k+\frac{1}{(n-1)!}\int_{-\infty}^\infty\mu(t)\spar{\int_0^x(x-z)^{n-1}\exp(-tz)\diff z}\diff t\tag{Fubini's}\\
&=\sum_{k=0}^{n-1}\frac{\Phi^{(k)}(0)}{k!}x^k+\frac{1}{(n-1)!}\int_{-\infty}^\infty\mu(t)\cdot t^{-n}\exp(-tx)\spar{\int_0^{tx} z^{n-1}\exp(z)\diff z}\diff t.\tag{change of variable}
\end{align*}
The indefinite integral $\int z^{n-1}\exp(z)\diff z=(-1)^{n-1}\Gamma(n,-z)$, where $\Gamma$ denotes the upper incomplete gamma function. By its sum representation, for $n\in\N_+$ we have $\Gamma(n,-z)=(n-1)!\exp(z)\sum_{k=0}^{n-1}\frac{(-z)^k}{k!}$. Therefore
\begin{equation*}
\Phi(x)=\sum_{k=0}^{n-1}\frac{\Phi^{(k)}(0)}{k!}x^k+\int_{-\infty}^\infty\mu(t)\phi_t(x)\diff t,
\end{equation*}
with the generalized kernel function $\phi_t$ defined as
\begin{equation*}
\phi_t(x)=\rpar{-\frac{1}{t}}^n\spar{\exp(-tx)-\sum_{k=0}^{n-1}\frac{(-tx)^k}{k!}}.
\end{equation*}
To prove Eq.~\eqref{eq:kernel_ineq} for such generalized $\phi_t$, we have to require an even $n$ as otherwise the multiplier on $\exp(-tx)$ is negative. Furthermore, if $n\geq 4$, $\phi_t(x)$ contains the $x^3$ component which is nonconvex. 

\vspace{0.5em}\noindent\textbf{Bernstein-Widder theorem. }The assumption of Theorem~\ref{thm:main_technical} might be interesting from a certain mathematical perspective, and we briefly discuss it here. A function $f:(0,\infty)\rightarrow[0,\infty)$ is called completely monotone if (i) it is infinitely differentiable, and (ii) for all $x>0$ and $n\geq 1$ we have
\begin{equation*}
(-1)^nf^{(n)}(x)\geq 0.
\end{equation*}
Given such a function $f$, $f(-x)$ is absolutely monotone on $(-\infty,0)$ in the sense that for all $x<0$ and $n\geq 1$, $f^{(n)}(x)\geq 0$. 

The Bernstein-Widder theorem characterizes a variant of Theorem~\ref{thm:main_technical}'s assumption when only the one-sided (rather than two-sided) Laplace transform is considered. It states that a function $f$ is completely monotone if and only if there exists a non-negative finite Borel measure $\mu$ on $[0,\infty)$ such that
\begin{equation*}
f(x)=\int_0^\infty\exp(-tx)\diff\mu(t),\quad \forall x>0.
\end{equation*}
Ref.~\cite{widder1946laplace} is a classical reference on this topic, and it remains open whether such results can bring concrete benefits to the understanding of Eq.~\eqref{eq:def-jensen}, or its applications in online learning. 


\subsection{Conjecture on even degree monomials}
\label{subsec:alternative_perspective}
An alternative perspective on Eq.~\eqref{eq:def-jensen} emerges when considering even degree monomials. We conducted numerical experiments with randomly sampled $S$ and $G$, and did not find any counterexample to the following proposition.
\begin{conjecture}
\label{prop:even_poly_conjecture}
For all $k\in\N_+$, Eq.~\eqref{eq:def-jensen} holds if $\Phi(x)=x^{2k}$. 
\end{conjecture}

This is true for $k=1$ and $2$, but the analysis with larger $k$ faces the difficulty discussed after Example~\ref{example:x_four}. It is important to note that the proof does not follow from Theorem~\ref{thm:main_technical}, as $\Phi(x) = x^{2k}$ violates the required condition on its second derivative. To see this, suppose for the sake of contradiction that $\Phi''(x) = 2k(2k-1)x^{2k-2}$ could be expressed as the Laplace transform of a non-negative function $\mu$. The $(2k+2)$-th derivative of $\Phi(x)$ is zero, which implies the $2k$-th derivative of $\Phi''(x)$ is also zero:
$$
\Phi^{(2k+2)}(x) = \frac{\diff^{2k}}{\diff x^{2k}}\Phi''(x) = \int_{-\infty}^{\infty} (-t)^{2k}\mu(t)e^{-tx} \diff t = \int_{-\infty}^{\infty} t^{2k}\mu(t)e^{-tx} \diff t = 0.
$$
Since $t^{2k}\mu(t)$ is a non-negative function, its Laplace transform can only be the zero function if $t^{2k}\mu(t)=0$ for almost every $t$. This implies $\mu(t)=0$ almost everywhere, which in turn means $\Phi''(x)=0$. This would require $\Phi(x)$ to be an affine function, which contradicts $\Phi(x)=x^{2k}$ for $k \ge 1$. Thus, a different proof technique is required for the above conjecture.

Conjecture~\ref{prop:even_poly_conjecture} is significant for two reasons. First, if true, it would imply that Eq.~\eqref{eq:def-jensen} holds for any $\Phi(x)$ that can be expressed as a power series of even powers with non-negative coefficients, such as $\Phi(x) = \exp(x^2)$. In the context of online learning, this is the foundation of parameter-free potential functions. Second, Conjecture~\ref{prop:even_poly_conjecture} could provide valuable insights for the disentanglement of interleaving matrix products, which is a major theme of matrix analysis \cite{lee2021golden,tian2025MatrixAnalysis}. More concretely, with $\eps=1$ and $\Phi(x)=x^{2k}$, Eq.~\eqref{eq:def-jensen} reduces to
\begin{align}
\label{eq:jensen-2k-sim} 
    \tr[(S+G)^{2k}] \le \tr\left[ \sum_{\text{odd }j=1}^{2k-1} \binom{2k}{j} G S^{2k-j} + \sum_{\text{even }j=0}^{2k} \binom{2k}{j} S^{2k-j} \right].
\end{align}
The right-hand side might be viewed as upper-bounding the binomial expansion of $\tr[(S+G)^{2k}]$, where terms are regrouped and the nonlinear dependence on $G$ is eliminated using the condition $\|G\|_\op\le 1$.

As for proving Conjecture~\ref{prop:even_poly_conjecture}, a natural direction is therefore extending the disentangling lemmas in the literature \cite{lee2021golden,tian2025MatrixAnalysis} (i.e., generalizations of Lemma~\ref{lemma:disentangle}) to more complicated entanglement settings. However, such an extension appears elusive. Here is a notable negative result: even when only considering real PSD matrices $S$ and $G$, Plevnik~\cite{plevnik2016matrix} provided a counterexample where $\tr[S^4GSG^4] > \tr[S^5G^5]$, thereby showing that an inequality of the form $\mathrm{Re}\,\tr[S^{p_1}G^{q_1}\cdots S^{p_k}G^{q_k}] \le \tr[S^{\sum p_i}G^{\sum q_i}]$ does not hold for all arrangements of non-negative exponents.

On the positive side, we establish the following disentangling lemma for interleaving matrix products, where the total exponent of each matrix is even. To the best of our knowledge, we are unaware of such result in the literature. The proof is based on a possibly interesting inductive argument. 

\begin{lemma}\label{lem:partial_result}
For any Hermitian matrices $S$ and $G$ with $\|G\|_\op\le 1$ and integers $k\ge 1$, $0< l \le k$, let $X_0, X_1, \ldots, X_{2k-1}\in\{ G,S\}$ such that the number of $G$ matrices is $\#\{j\mid X_j = G\}=2l$. Then we have
\[
|\tr[X_0X_1\cdots X_{2k-1}]| \le \tr[S^{2k-2l}].
\]
\end{lemma}

\begin{proof}
Let $\mathcal{P}_{2l}$ denote the set of all matrix products of length $2k$ containing $2l$ instances of $G$ and $2k-2l$ instances of $S$. Let $P_* = X_0X_1\cdots X_{2k-1}$ be a product in $\mathcal{P}_{2l}$ that achieves the maximum absolute value of the trace. If $|\tr[P_*]| = 0$, the lemma is trivially true as $\tr[S^{2k-2l}] \ge 0$ (since $S$ is Hermitian, $S^{2k-2l}$ is PSD). We therefore proceed with the case $|\tr[P_*]| > 0$.

First, we establish that the product can be arranged appropriately without changing its trace. Let $c(m)$ be the number of $G$ matrices in the window of $k$ consecutive matrices starting at index $m$ (indices are mod $2k$). As this window slides, the count $c(m)$ changes by at most 1 at each step. Since the average count is $l$, there must exist an index $m$ where $c(m)=l$. Due to the trace's cyclic invariance, we can consider the product starting from $X_m$. Thus, we can assume the first half, $P_1 = X_0 \cdots X_{k-1}$, and the second half, $P_2 = X_k \cdots X_{2k-1}$, each contains exactly $l$ instances of $G$.

Furthermore, if both $X_0=S$ and $X_k=S$, such an $l$ - $l$ split of all $2l$ instances of $G$ is maintained if we cyclically shift the matrix product to begin with $X_1$, and meanwhile, the trace stays unchanged. We can repeat this shifting process. Since there are $G$ matrices in the product, this process must terminate, returning a sequence where either $X_0=G$ or $X_k=G$. The proof proceeds by assuming the former case, as the argument for the latter is symmetric. Therefore, we can assume without loss of generality that our sequence $X_0, \dots, X_{2k-1}$ is arranged such that the first and the second half each has $l$ matrices $G$, and $X_0=G$.

By the Cauchy-Schwarz inequality for the trace inner product $\langle A, B \rangle = \tr(A^* B)$, we have $|\tr(A^*B)|^2 \le \tr(A^*A)\tr(B^*B)$. Let $A = P_2^*$ and $B = P_1$. Then $A^*B = P_2P_1$, and by the cyclic property, $\tr(P_2P_1) = \tr(P_1P_2) = \tr(P_*)$. The inequality becomes:
\begin{align*}
    |\tr[P_*]|^2 &= |\tr[P_1 P_2]|^2 \\
    &\le \tr[P_1^* P_1] \tr[P_2 P_2^*] \\
    &= \tr[(X_0\cdots X_{k-1})^*(X_0\cdots X_{k-1})]\tr[(X_k\cdots X_{2k-1})(X_k\cdots X_{2k-1})^*].
\end{align*}
Since all matrices are Hermitian ($X_j^* = X_j$), this becomes:
\begin{align*}
    |\tr[P_*]|^2 &\le \tr[X_{k-1}\cdots X_1 X_0^2 X_1 \cdots X_{k-1}] \tr[X_k \cdots X_{2k-1} X_{2k-1} \cdots X_k].
\end{align*}
Now, consider the term $P'_2 = (X_k \cdots X_{2k-1})(X_{2k-1} \cdots X_k)$. This is a product of length $2k$ containing $2l$ instances of $G$, so $P'_2 \in \mathcal{P}_{2l}$. Since $P'_2 = P_2P_2^*$, it is PSD and its trace is real and non-negative. Because $P_*$ was chosen to have the maximum absolute trace, we must have $\tr[P'_2] = |\tr[P'_2]| \le |\tr[P_*]|$. This gives:
\[
|\tr[P_*]|^2 \le \tr[X_{k-1}\cdots X_1 G^2 X_1 \cdots X_{k-1}] |\tr[P_*]|.
\]
Since we consider the case where $|\tr[P_*]| > 0$, we can divide by it to obtain:
\begin{align*}
    |\tr[P_*]| &\le \tr[X_{k-1}\cdots X_1 G^2 X_1 \cdots X_{k-1}]\\
    &= \tr[G^2(X_1 \cdots X_{k-1})(X_{k-1}\cdots X_1)]\\
    &\le \tr[(X_1 \cdots X_{k-1})(X_{k-1}\cdots X_1)],
\end{align*}
where the second line follows from $G^2 \preceq I$ and $(X_1 \cdots X_{k-1})(X_{k-1}\cdots X_1)$ is PSD. As $(X_1 \cdots X_{k-1})(X_{k-1}\cdots X_1)$ contains $2(l-1)$ instances of $G$, we can apply this reduction argument inductively on $l$. At each step, we remove two instances of $G$. After applying the reduction $l$ times, we are left with a product of $2k-2l$ instances of $S$. Therefore,
\[
|\tr[X_0X_1\cdots X_{2k-1}]| \le \tr[S^{2k-2l}]. \qedhere
\]
\end{proof}

A simple corollary is that for Hermitian PSD matrices, we can apply Lemma~\ref{lem:partial_result} to their square roots to obtain an inequality of the Lieb–Thirring type. 
\begin{corollary}
For any Hermitian PSD matrices $S$ and $G$ with $\|G\|_\op\le 1$ and integers $k\ge 1$, $0< l \le k$, let $X_0, X_1, \ldots, X_{k-1}\in\{ G,S\}$ such that the number of $G$ matrices is $\#\{j\mid X_j = G\}=l$. Then we have
\[
|\tr[X_0X_1\cdots X_{k-1}]| \le \tr[S^{k-l}].
\]
\end{corollary}

We remark that while Lemma~\ref{lem:partial_result} provides an upper bound for a large class of interleaving matrix products, it alone is insufficient to prove Conjecture~\ref{prop:even_poly_conjecture}, as the binomial expansion of $\tr[(S+G)^{2k}]$ also contains terms with an odd number of $G$ matrices.

\section{Algorithm and Analysis}\label{sec:algorithm}

Based on the techniques from Section~\ref{sec:matrix_Jensen}, this section introduces our matrix LEA algorithm as well as its regret bound. We will show that the Laplace transform characterization of the one-sided Jensen's trace inequality (Theorem~\ref{thm:main_technical} and Corollary~\ref{thm:phi_itself}) connects nicely to the inherent structure of parameter-free potential functions. Omitted proofs for this section are presented in Appendix~\ref{section:reg_upper_proofs}.

\vspace{0.5em}\noindent\textbf{Lifting the constraint. }Recall that matrix LEA is essentially OLO on the spectraplex. As the preparatory step, we first present a rate-preserving reduction (Algorithm~\ref{alg:reduction}) to an unconstrained OLO problem on $\mathbb{H}_{d\times d}$, the space of all Hermitian matrices. This is a mild generalization of relevant techniques in vector LEA \cite{luo2015achieving,orabona2016coin,cutkosky2018black} to the matrix setting. 

\begin{algorithm*}[!ht]
\caption{Reducing matrix LEA to OLO on $\mathbb{H}_{d\times d}$.\label{alg:reduction}}
\begin{algorithmic}[1]
\REQUIRE A base OLO algorithm $\A$ on the domain $\mathbb{H}_{d\times d}$, which at each time step outputs a prediction $\tilde X_t\in\mathbb{H}_{d\times d}$ and then takes in a loss matrix $G_t\in\mathbb{H}_{d\times d}$. 
\FOR{$t=1,2,\ldots,$}
\STATE Query the base algorithm $\A$ for its $t$-th prediction $\tilde X_t\in\mathbb{H}_{d\times d}$, and perform the eigen-decomposition $\tilde X_t=\sum_{i=1}^d\lambda_{t,i}v_{t,i}v_{t,i}^\herm$, where $\lambda_{t,i}\in\R$ and $v_{t,i}\in\C^d$.

\STATE For matrix LEA, predict
\begin{equation*}
X_t=\frac{\sum_{i=1}^d\max\{0,\lambda_{t,i}\} v_{t,i} v_{t,i}^\herm}{\sum_{i=1}^d\max\{0,\lambda_{t,i}\}}\in\Delta_{d\times d}.
\end{equation*}
\STATE Receive the loss matrix $G_t\in\mathbb{H}_{d\times d}$, and process it by the following two-step projection:
\begin{itemize}
\item Compute an intermediate quantity $\bar G_t=G_t-\inner{G_t}{X_t}I_d$.
\item If the eigenvalues $\lambda_{t,1},\ldots,\lambda_{t,d}\geq 0$, define the surrogate loss matrix $\tilde G_t=\bar G_t$. Otherwise, define
\begin{equation}\label{eq:second_projection}
\tilde G_t=
\bar G_t-\min\left\{0,\inner{\bar G_t}{U_t}\right\}U_t,
\end{equation}
where $U_t=\frac{\sum_{i=1}^d\min\{0,\lambda_{t,i}\}v_{t,i} v_{t,i}^\herm}{\abs{\sum_{i=1}^d\min\{0,\lambda_{t,i}\}}}$. In both cases $\tilde G_t\in\mathbb{H}_{d\times d}$ by construction.
\end{itemize}
\STATE Send $\tilde G_t$ to the base algorithm $\A$ as its $t$-th loss matrix. 
\ENDFOR
\end{algorithmic}
\end{algorithm*}

\begin{restatable}{lemma}{reduction}\label{lemma:reduction}
Algorithm~\ref{alg:reduction} satisfies the following two conditions: $\norm{\tilde G_t}_\op\leq 2\norm{G_t}_\op$, and $\inner{G_t}{X_t-X}\leq \inner{\tilde G_t}{\tilde X_t-X}$ for all $X\in\Delta_{d\times d}$. 
\end{restatable}

The proof of Lemma~\ref{lemma:reduction} is given in Appendix~\ref{app:proof_reduction}. The implication is that the regret of Algorithm~\ref{alg:reduction} is bounded by the regret of the underlying base algorithm $\A$. Consequently, the rest of the section will focus on the domain $\mathbb{H}_{d\times d}$.

\subsection{Unconstrained algorithm on \texorpdfstring{$\mathbb{H}_{d\times d}$}{H_d}}\label{sec:unconstraint_algo}

Next, we design the unconstrained base algorithm $\A$ required by Algorithm~\ref{alg:reduction}. For notational convenience we will drop the tilde from this point, meaning that the $X_t$ and $G_t$ next are rigorously $\tilde X_t$ and $\tilde G_t$ in Algorithm~\ref{alg:reduction}. By Lemma~\ref{lemma:reduction} and the problem setting (Definition~\ref{def:matrix_lea}), such a $G_t$ satisfies $\norm{G_t}_\op\leq 2l$. We will also write $I_d$ as $I$ since the dimensionality is clear.  

Our unconstrained algorithm is based on the matrix version of the potential method (Algorithm~\ref{alg:unconstrained}). Here, the function $\Phi_t$ defined on $\R$ is applied to Hermitian matrices in the standard spectral manner: for any $X\in\mathbb{H}_{d\times d}$ with eigen-decomposition $X=\sum_{i=1}^d\lambda_iv_iv_i^\herm$, we define $\Phi_t(X)\defeq \sum_{i=1}^d\Phi_t(\lambda_i)v_iv_i^\herm\in\mathbb{H}_{d\times d}$. 

\begin{algorithm*}[!ht]
\caption{Unconstrained algorithm on $\mathbb{H}_{d\times d}$.\label{alg:unconstrained}}
\begin{algorithmic}[1]
\REQUIRE A convex potential function $\Phi_t:\R\rightarrow\R$, dependent on $t\in\N_+$. 
\STATE Define the constant $\eps=2l$. Initialize the matrix $S_1=0\in\mathbb{H}_{d\times d}$.
\FOR{$t=1,2,\ldots,$}
\STATE Output
\begin{equation}\label{eq:potential_update}
X_t=\frac{1}{2\eps}\spar{\Phi_t\rpar{S_t+\eps I}-\Phi_t\rpar{S_t-\eps I}}\in\mathbb{H}_{d\times d}. 
\end{equation}

\STATE Receive the loss matrix $G_t\in\mathbb{H}_{d\times d}$ satisfying $\norm{G_t}_\op\leq \eps$. 
\STATE Let $S_{t+1}=S_t-G_t$.
\ENDFOR
\end{algorithmic}
\end{algorithm*}

The intuition is the following. The $X_t$ defined in Eq.~\eqref{eq:potential_update} approximates the evaluation of the derivative function $\Phi'_t$ at $S_t=-\sum_{i=1}^{t-1}G_i$, which is essentially the dual update of FTRL discussed in Section~\ref{sec:prelim_lea_mmwu}. In fact, the MMWU update Eq.~\eqref{eq:mmwu_update} can be recovered by choosing $\Phi_t(s)=\exp(\eta_t s)$ in Algorithm~\ref{alg:unconstrained}, replacing the discrete derivative in Eq.~\eqref{eq:potential_update} by the exact derivative, and passing the obtained algorithm through the $\Delta_{d\times d}$-to-$\mathbb{H}_{d\times d}$ reduction (Algorithm~\ref{alg:reduction}). As shown in a number of earlier works \cite{luo2015achieving,orabona2016coin,harvey2020optimal,zhang2022pde}, the discrete derivative is crucial for the use of general potential functions in this framework. Roughly speaking, the reason is that online learning algorithms can be regarded as the discretization of certain continuous-time decision rules against stochastic processes, and the discrete derivative provides the right amount of ``robustness'' against the more challenging discrete-time non-stochastic adversaries, through the use of the Jensen's inequality. 

Without specifying the potential function $\Phi_t$, we provide the following master theorem on Algorithm~\ref{alg:unconstrained}. Two conditions on $\Phi_t$ are required. The idea is that the first condition only concerns the one-dimensional behavior of $\Phi_t$ which is traditionally the bottleneck in (scalar or diagonal) parameter-free online learning, but now this is standard in the literature for the potential functions we consider. The second condition is the main bottleneck for matrix LEA, which highlights the crucial role of our one-sided Jensen's trace inequality from Section~\ref{sec:matrix_Jensen}. 

\begin{theorem}[Master regret bound]\label{thm:olo_master}
In Algorithm~\ref{alg:unconstrained}, assume the potential function $\Phi_t:\R\rightarrow\R$ satisfies the following two conditions: 

\begin{enumerate}
\item For all $t\in\N_+$ and $s\in\R$, 
\begin{equation}\label{eq:olo_master_1}
\frac{1}{2}\spar{\Phi_{t+1}(s+\eps)+\Phi_{t+1}(s-\eps)}\leq\Phi_t(s).
\end{equation}
\item For all $t\in\N_+$, $S\in\mathbb{H}_{d\times d}$ and $G\in\mathbb{H}_{d\times d}$ satisfying $\norm{G}_\op\leq \eps$,
\begin{equation}\label{eq:olo_master_2}
\tr\spar{\Phi_{t}(S-G)}\leq \tr\spar{\frac{\eps I-G}{2\eps}\Phi_{t}(S+\eps I)+\frac{\eps I+G}{2\eps }\Phi_{t}(S-\eps I)}.
\end{equation}
\end{enumerate}

\noindent Then, Algorithm~\ref{alg:unconstrained} guarantees that for all $T\in\geq 2$ and $X\in\mathbb{H}_{d\times d}$ with eigenvalues $\lambda_1,\ldots,\lambda_d$, 
\begin{equation*}
\sum_{t=1}^T\inner{G_t}{X_t-X}\leq \inner{G_1}{X_1}+\tr\spar{\Phi_{1}(-G_1)}+\sum_{i=1}^d\Phi^*_T(\lambda_i).
\end{equation*}
Here $\Phi_T^*(\lambda)\defeq\sup_{g\in\R} \spar{g\lambda-\Phi_T(g)}$ denotes the Fenchel conjugate of the function $\Phi_T$. 
\end{theorem}

\begin{proof}[Proof of Theorem~\ref{thm:olo_master}] Consider $t\geq 2$. We start by plugging $S\leftarrow S_t$ and $G\leftarrow G_t$ into Eq.~\eqref{eq:olo_master_2}, which yields
\begin{align*}
\tr\spar{\Phi_{t}(S_{t+1})}&\leq \tr\spar{\frac{\eps I-G_t}{2\eps}\Phi_{t}(S_t+\eps I)+\frac{\eps I+G_t}{2\eps }\Phi_{t}(S_t-\eps I)}\\
&=\frac{1}{2}\tr\spar{\Phi_t(S_t+\eps I)+\Phi_t(S_t-\eps I)}-\inner{G_t}{X_t}\tag{Eq.~\eqref{eq:potential_update}}\\
&\leq \tr\spar{\Phi_{t-1}(S_t)}-\inner{G_t}{X_t}.\tag{Eq.~\eqref{eq:olo_master_1}; $S_t$ commutes with $S_t\pm\eps I$}
\end{align*}
Taking the summation over $t$, we obtain the total loss bound
\begin{equation*}
\sum_{t=1}^T\inner{G_t}{X_t}\leq \underbrace{\inner{G_1}{X_1}+\tr\spar{\Phi_{1}(-G_1)}}_{\eqdef \bigstar}-\tr\spar{\Phi_T\rpar{-\sum_{t=1}^TG_t}}.
\end{equation*}

Next, we use a simple convex duality technique due to Ref.~\cite{mcmahan2014unconstrained}. For all $X\in\mathbb{H}_{d\times d}$,
\begin{align*}
\sum_{t=1}^T\inner{G_t}{X_t-X}&\leq \bigstar+\inner{-\sum_{t=1}^TG_t}{X}-\tr\spar{\Phi_T\rpar{-\sum_{t=1}^TG_t}}\\
&\leq \bigstar+\sup_{G\in\mathbb{H}_{d\times d}}\left\{\inner{G}{X}-\tr\spar{\Phi_T\rpar{G}}\right\}.
\end{align*}
Due to von Neumann's trace inequality (Lemma~\ref{lemma:von_neumann}), the supremum on the RHS is obtained when $G$ and $X$ commute, which translates the optimization over the matrix domain $\mathbb{H}_{d\times d}$ to the domain $\R$ of eigenvalues. That is, with the eigenvalues of $X$ denoted by $\lambda_1,\ldots,\lambda_d$, 
\begin{equation*}
\sup_{G\in\mathbb{H}_{d\times d}}\left\{\inner{G}{X}-\tr\spar{\Phi_T\rpar{G}}\right\}=\sum_{i=1}^d\sup_{g\in\R} \left\{g\lambda_i-\Phi_T(g)\right\}=\sum_{i=1}^d\Phi^*_T(\lambda_i).
\end{equation*}
Plugging it back completes the proof. 
\end{proof}

\subsection{Potentials and their regret bounds}

Now we consider the instantiation of our algorithm with parameter-free potential functions. We start from a simpler but suboptimal choice, and then extend this argument to the related ``erfi potential'' which is our main result. 

\vspace{0.5em}\noindent\textbf{Exp-square potential. }The following is a classical parameter-free potential function due to Ref.~\cite{mcmahan2014unconstrained} and further studied by a number of subsequent works \cite{luo2015achieving,orabona2016coin,mhammedi2020lipschitz}. 
\begin{equation}
\Phi^\expsq_t(s)\defeq\frac{\eps}{d\sqrt{t}}\exp\rpar{\frac{s^2}{2\eps^2 t}},\quad\forall s\in\R.\label{eq:exp_square}
\end{equation}
It satisfies the first condition in Theorem~\ref{thm:olo_master} due to Ref.~\cite[Lemma~B.3]{zhang2022pde}. 

A key intermediate result of this paper is the following characterization of $\Phi^\expsq_t$ as the Laplace transform of a dilated Gaussian density. 

\begin{lemma}[$\Phi^\expsq_t$ as Laplace transform]\label{lemma:phi_laplace}
The function $\Phi^\expsq_t$ defined in Eq.~\eqref{eq:exp_square} satisfies
\begin{equation*}
\Phi^\expsq_t(s)=\int_{-\infty}^\infty\mu(z)\exp(-zs)\diff z,
\end{equation*}
where
\begin{equation*}
\mu(z)=\frac{\eps^2}{\sqrt{2\pi}d}\exp\rpar{-\frac{1}{2}\eps^2tz^2}.
\end{equation*}
\end{lemma}

\begin{proof}[Proof of Lemma~\ref{lemma:phi_laplace}]
Recall the classical Gaussian integral: for all $a>0$, $b\in\R$, we have
\begin{equation*}
\int_{-\infty}^\infty\exp\rpar{-a(x+b)^2}\diff x=\sqrt{\frac{\pi}{a}}. 
\end{equation*}

Next, consider the two-sided Laplace transform of $\mu(z)=\exp(-cz^2)$ for an arbitrary $c>0$,
\begin{align*}
\int_{-\infty}^\infty\mu(z)\exp\rpar{-zx}\diff z&=\int_{-\infty}^\infty\exp\rpar{-cz^2-zx}\diff z\\
&=\exp\rpar{\frac{x^2}{4c}}\int_{-\infty}^\infty\exp\rpar{-c\rpar{z+\frac{x}{2c}}^2}\diff z\\
&=\sqrt{\frac{\pi}{c}}\exp\rpar{\frac{x^2}{4c}},
\end{align*}
where the last line follows from the above Gaussian integral. The proof is complete by letting $c\leftarrow \frac{\eps^2 t}{2}$ and scaling both sides by $\frac{\eps^2}{\sqrt{2\pi}d}$. 
\end{proof}

Since the dilated Gaussian density is positive, we can then invoke Corollary~\ref{thm:phi_itself} to show that $\Phi^\expsq_t$ also satisfies the second condition in Theorem~\ref{thm:olo_master}. Combining it with the $\Delta_{d\times d}$-to-$\mathbb{H}_{d\times d}$ reduction (Lemma~\ref{lemma:reduction}) gives us the final regret bound induced by $\Phi^\expsq_t$. 

\begin{restatable}[Regret bound from $\Phi^\expsq_t$]{theorem}{expsquare}\label{thm:lea_worse}
Consider the instantiation of Algorithm~\ref{alg:unconstrained} with the potential function $\Phi^\expsq_t$ from Eq.~\eqref{eq:exp_square}, and use that as the base algorithm to define an instance of Algorithm~\ref{alg:reduction}. This is a matrix LEA algorithm that guarantees
\begin{equation*}
\reg_T(X)\leq 2\sqrt{2}l\sqrt{T\cdot S(X||d^{-1}I_d)}+4\sqrt{2}l\sqrt{T\log T}+2\sqrt{e}l,
\end{equation*}
for all $T\geq 1$ and $X\in\Delta_{d\times d}$. 
\end{restatable}

Despite the desirable $\sqrt{T\cdot S(X||d^{-1}I_d)}$ term, such a result does not exactly achieve our goal due to the additional $\sqrt{T\log T}$ in the bound, meaning that it is only an improvement over MMWU when $d\gg T$. To fix this issue we consider the following closely related potential function due to Ref.~\cite{harvey2020optimal} and further studied by Refs.~\cite{zhang2022pde,harvey2024continuous}. 

\vspace{0.5em}\noindent\textbf{Erfi potential. }Given $\Phi^\expsq_t$ from Eq.~\eqref{eq:exp_square}, define
\begin{align}
\Phi^\erfi_t(s)&\defeq\frac{\eps\sqrt{t}}{d}\spar{2\int_0^{\frac{s}{\sqrt{2\eps^2t}}}\rpar{\int_0^u\exp(x^2)\diff x }\diff u-1}\label{eq:erfi_potential}\\
&=\frac{\sqrt{2}s}{\eps d}\int_0^{\frac{s}{\sqrt{2\eps^2 t}}}\exp(u^2)\diff u-\frac{\sqrt{t}}{d}\exp\rpar{\frac{s^2}{2\eps^2 t}}.\nonumber\tag{integration by parts}
\end{align}
Notice that the function $f(x)=\int_0^x\exp(u^2)\diff u$ is the imaginary error function (scaled by a constant). Just like $\Phi^\expsq_t$, $\Phi^\erfi_t$ also satisfies the first condition in Theorem~\ref{thm:olo_master} due to Ref.~\cite[Lemma~3.10]{harvey2023optimal}. 

Importantly, it is known that the second derivative $\left\{\Phi^\erfi_t\right\}'(s)=\frac{1}{\eps^2}\Phi^\expsq_t(s)$ \cite[Appendix~B.3]{zhang2022pde}, therefore by Theorem~\ref{thm:main_technical} and Lemma~\ref{lemma:phi_laplace}, $\Phi^\erfi_t$ satisfies the second condition in Theorem~\ref{thm:olo_master}. 
Combining it with the Fenchel conjugate computation \cite[Theorem~4]{zhang2022pde} leads to the main result of this paper. 

\begin{restatable}[Main; regret bound from $\Phi^\erfi_t$]{theorem}{main}\label{thm:lea_main}
Consider the instantiation of Algorithm~\ref{alg:unconstrained} with the potential function $\Phi^\erfi_t$ from Eq.~\eqref{eq:erfi_potential}, and use that as the base algorithm to define an instance of Algorithm~\ref{alg:reduction}. This is a matrix LEA algorithm that guarantees
\begin{equation*}
\reg_T(X)\leq l\sqrt{T}\spar{\sqrt{8 S(X||d^{-1}I_d)}+6+2\sqrt{2}}=\calO\rpar{\sqrt{T\cdot S(X||d^{-1}I_d)}},
\end{equation*}
for all $T\geq 1$ and $X\in\Delta_{d\times d}$. 
\end{restatable}

Theorem~\ref{thm:lea_main} is a substantial improvement over the $\calO\rpar{\sqrt{T\log d}}$ regret bound of MMWU: in the worst case of $X$ our regret bound is also $\calO\rpar{\sqrt{T\log d}}$, but when $X$ is ``easy'', such as when $S(X||d^{-1}I_d)=\calO(1)$, our regret bound improves to $\calO(\sqrt{T})$. Downstream benefits in quantum learning theory are discussed in Section~\ref{sec:quantum}. 

Meanwhile, as discussed in the introduction, the computational complexity of a matrix LEA algorithm is an important consideration. 

\begin{remark}[Computational complexity]\label{remark:computation}
Both algorithms above have the same time and memory complexity as MMWU (up to constant multiplicative factors), even in the setting with an eigen-decomposition oracle. In particular, just like MMWU, in the $t$-th round they output a spectral function of the matrix $\sum_{i=1}^{t-1}G_i$, and with an eigen-decomposition oracle such a spectral function can be computed with $\calO(d)$ time and memory. 
\end{remark}

Another remark is that our $\Delta_{d\times d}$-to-$\mathbb{H}_{d\times d}$ reduction (Algorithm~\ref{alg:reduction}) consists of two independent steps, (i) relaxing the constraint of $\norm{X_t}_\tr=1$, and (ii) relaxing the constraint of $X_t\succeq 0$. If we only keep the PSD constraint, then the algorithm from Theorem~\ref{thm:lea_main} becomes a state-of-the-art algorithm for predicting Hermitian PSD matrices. Such a problem has been a canonical example of online learning on Banach spaces (equipped with the trace norm). Here, with the only constraint being $X_t\in\mathbb{H}_{d\times d},X_t\succeq 0$, Ref.~\cite[Theorem~6]{foster2017parameter} presents a model-selection-based algorithm achieving
\begin{equation*}
\reg_T(X)=\calO\rpar{\rpar{1+\norm{X}_\tr}\sqrt{T\log d\log\spar{\rpar{1+\norm{X}_\tr}T}}},\quad \forall X\in\mathbb{H}_{d\times d}, X\succeq 0,
\end{equation*}
and the per-round time complexity is $\calO(T)$. There has been substantial progress on Banach space online learning since then, and a current folklore is that one could instantiate the generic reduction of Ref.~\cite[Section~3]{cutkosky2018black} with (i) MMWU and (ii) the one-dimensional learner of Ref.~\cite[Section~4]{zhang2022pde}, which achieves
\begin{equation*}
\reg_T(X)=\calO\rpar{\sqrt{T}\rpar{1+\norm{X}_\tr\sqrt{\log \spar{d\rpar{1+\norm{X}_\tr}}}}},\quad \forall X\in\mathbb{H}_{d\times d}, X\succeq 0,
\end{equation*}
and the time complexity matches that of MMWU. This is the state of the art. 

The aforementioned PSD variant of our algorithm matches this state of the art. From the proof of Theorem~\ref{thm:lea_main}, we see that it guarantees
\begin{align*}
\reg_T(X)&=\calO\rpar{\sqrt{T}\rpar{1+\sum_{i=1}^d\lambda_i\sqrt{\log\rpar{1+d\lambda_i}}}}\tag{$\lambda_{1:d}$ are eigenvalues of $X$}\\
&=\calO\rpar{\sqrt{T}\rpar{1+\norm{X}_\tr\sqrt{\log\rpar{1+d\norm{X}_\op}}}},\quad \forall X\in\mathbb{H}_{d\times d}, X\succeq 0,
\end{align*}
which is the same as the above folklore. Besides, as the first line is required to achieve comparator adaptivity in matrix LEA, we also see that existing techniques in Banach space online learning are insufficient for the objectives of this paper. 

\subsection{Our algorithm as Gaussian ensemble}\label{sec:ensemble}

A byproduct of our Laplace transform technique is an interpretation of our algorithms as Gaussian ensembles. We now elucidate this interpretation which helps explain how the oddly looking parameter-free potential functions relate to the learning rate tuning issue of MMWU (Section~\ref{sec:prelim_lea_mmwu}). 

We will focus on $\Phi^\expsq_t$ from Eq.~\eqref{eq:exp_square}. First, by reformulating Lemma~\ref{lemma:phi_laplace}, $\Phi^\expsq_t$ is equivalent to the Gaussian expectation of a parameterized ``base'' potential: with $\phi(x)\defeq\sqrt{\frac{2}{\pi}}\exp\rpar{-\frac{x^2}{2}}$ being two times the probability density function (PDF) of the standard normal distribution, 
\begin{equation*}
\Phi^\expsq_t(s)=\int_{0}^\infty\phi(z)\Phi^\cosh_{t}(s;z)\diff z,
\end{equation*}
where the ``cosh potential'' $\Phi^\cosh_{t}(\cdot;z)$ parameterized by $z\geq 0$ is defined as
\begin{equation*}
\Phi^\cosh_{t}(s;z)\defeq\frac{\eps}{d\sqrt{t}}\cosh\rpar{\frac{zs}{\eps\sqrt{t}}}. 
\end{equation*}

This is closely related to the exponential potential that MMWU relies on. Suppose $zs\gg \eps\sqrt{t}$, then $\Phi^\cosh_{t}(s;z)\approx \frac{\eps}{2d\sqrt{t}}\exp\rpar{\frac{zs}{\eps\sqrt{t}}}$, and applying it to Algorithm~\ref{alg:reduction} and Algorithm~\ref{alg:unconstrained} with the discrete derivative in Eq.~\eqref{eq:potential_update} replaced by the exact derivative gives us the matrix LEA prediction
\begin{equation*}
X_t\approx\frac{\exp\rpar{-\frac{z}{\eps\sqrt{t}}\sum_{i=1}^{t-1}G_i}}{\tr\exp\rpar{-\frac{z}{\eps\sqrt{t}}\sum_{i=1}^{t-1}G_i}}.
\end{equation*}
Comparing this to the MMWU update Eq.~\eqref{eq:mmwu_update} shows that $z$ serves the role of a learning rate scalar, and the oracle tuning of MMWU would set $z=\sqrt{S(X||d^{-1}I_d)}$. Intuitively, the takeaway is that our algorithm with the $\Phi^\expsq_t$ potential is essentially the expectation of MMWU-like base algorithms with Gaussian distributed learning rates. 

Let us compare this idea to the literature. In the vector setting of LEA, Koolen and van Erven \cite{koolen2015second} studied how to achieve comparator adaptive regret bounds by aggregating non-adaptive algorithms with respect to certain carefully designed, non-Gaussian priors. Revisiting this idea ten years later, we offer an intriguing new observation: $\Phi^\erfi_t$ and $\Phi^\expsq_t$ are both expectations with respect to the Gaussian prior, and their difference lies in the base potential being integrated. Compared to Ref.~\cite{koolen2015second}, this suggests that besides improving the prior, improving the base potential can also lead to substantial quantitative benefits. 

\section{Matching Lower Bounds for Regret and Memory}\label{sec:lower_bounds}

This section supplements our main result (Theorem~\ref{thm:lea_main}) with optimality analysis. Section~\ref{sec:regret_lower} presents a matching $\Omega\rpar{\sqrt{T\cdot S(X||d^{-1}I_d)}}$ regret lower bound, while Section~\ref{sec:mem_lower} presents a matching $\Omega(d^2)$ memory lower bound. The time complexity is a subtle matter which the present work does not consider. 

\subsection{Matching regret lower bound}\label{sec:regret_lower}

As the problem of matrix LEA is a strict generalization of vector LEA, regret lower bounds of the latter are automatically applicable to our setting. Ref.~\cite[Theorem~1]{negrea2021minimax} provides a comparator-dependent regret lower bound for vector LEA, showing that with respect to any comparator $u$ in the probability simplex $\Delta_d$, $\Omega\rpar{\sqrt{T\cdot \kl(u||d^{-1}\bm{1}_d)}}$ regret is essentially unavoidable. While this is already sufficient for our need, here we present a different argument (albeit with a worse constant) in order to keep the present work self-contained. Omitted proofs are presented in Appendix~\ref{section:regret_lower_proofs}. 

\begin{restatable}[Regret lower bound]{theorem}{reglower}\label{thm:reg_lower}
For the matrix LEA problem (Definition~\ref{def:matrix_lea}) with $l=1$, there exists absolute constants $c_1,c_2,c_3,C>0$ such that the following statement holds. For any $d\geq c_1$, $T\geq c_2 d^2$, $r\in[c_3,\log d]$ and any algorithm $\calA$, there exist
\begin{itemize}
\item an adversary; and
\item a comparator $X\in\Delta_{d\times d}$ satisfying $S(X||d^{-1}I_d)\leq r$,
\end{itemize}
such that the regret of $\calA$ with respect to $X$ satisfies
\begin{equation*}
\reg_T(X)\ge \sqrt{\frac{1}{3}T}\rpar{\sqrt{2r}-C}.
\end{equation*}
\end{restatable}

Comparing this to Theorem~\ref{thm:lea_main}, we conclude that our proposed matrix LEA algorithm has the order-optimal regret bound with respect to the comparator $X$. The gap to the lower bound is on the constants: the leading constants of our regret upper and lower bounds are $2\sqrt{2}$ and $\sqrt{\frac{2}{3}}$ respectively. 

We remark that Ref.~\cite[Theorem~1]{negrea2021minimax} has a better regret lower bound with the leading constant $\sqrt{2}$. This is also the optimal leading constant in the fixed-time setting, where the considered algorithm can possibly know the duration $T$ beforehand \cite{orabona2015optimal}. Our analysis follows a different, possibly simpler strategy assuming that the adversary samples actions from the continuous uniform distribution. Then, since the IID sum of such a distribution is unimodal, we directly invoke an elementary result to associate its order statistics with its tail probability \cite[Section~4.5]{david2004order}. Technically, the key intermediate result is the following. 

\begin{restatable}[Anti-concentration of unimodal order statistics]{lemma}{anti}\label{lemma:anti_concentration}
Let $\calD$ be a symmetric distribution on $\R$ with $\sigma\defeq \sqrt{\E_{X\sim\calD}[X^2]}>0$ and $\rho\defeq \E_{X\sim\calD}\spar{\abs{X}^3}<\infty$. In particular, we assume $\calD$ is unimodal: its cumulative distribution function (CDF) is convex on $\R_{<0}$ and concave on $\R_{>0}$. Let $\calD_n$ be the distribution of the sum of $n$ independent random variables, each with distribution $\calD$. 

Consider independent random variables $Z_1,\ldots,Z_d\sim\calD_n$, and for all $j\in[1:d]$, let $Z_{(j)}$ be their $j$-th order statistic, i.e., $Z_{(j)}$ is the $j$-th smallest element within $\{Z_1,\ldots,Z_d\}$. Then, for any positive integers $k$ satisfying $k\leq \frac{d+1}{\sqrt{2\pi}e^2}-1$ and $n$ satisfying $n\geq \frac{\rho^2}{\sigma^6}(d+1)^2$, we have
\begin{equation*}
\frac{1}{k}\sum_{j=d-k+1}^{d}\E\spar{Z_{(j)}}\geq \sigma\sqrt{n}\spar{\sqrt{2\log\frac{d}{\sqrt{2\pi}(k+1)}}-1}.
\end{equation*}
\end{restatable}

Lemma~\ref{lemma:anti_concentration} alone has the optimal leading constant $\sqrt{2}$, but the constraint of unimodality means that its conversion to the regret lower bound would suffer from $\sigma=\frac{1}{\sqrt{3}}$. 

We also remark that our regret upper bound is in the anytime setting. Here, the optimal leading constant of the regret bound remains an open problem in the literature \cite{harvey2024continuous}.

\subsection{Matching memory lower bound}\label{sec:mem_lower}

We prove the following matching memory lower bound for the matrix LEA problem. Note that recording each loss matrix from $G_1,\ldots,G_T$ or their sum already requires $\mathcal{O}(d^2)$ memory. For our result to be nontrivial, here we allow the learner to freely query the historical sequence of loss matrices $G_1,\ldots,G_{t-1}$ at each time step $t$, in order to bypass the recording memory overhead.

In our setting, the worst-case adversary can pick the loss matrix $G_t$ after observing the prediction $X_t$. In the case when $X_t$ is diagonal and the basis is known (computational basis), there exists a $\Omega(\min\{\eps^{-1}\log d,d\})$ memory lower bound for any vector LEA algorithm with $\mathcal{O}(\eps T)$ regret, due to~Ref.\cite{peng2023online}. Here we extend this result to the more general matrix LEA problem.

\begin{theorem}\label{thm:mem_lower}
For any $d,T\geq 1$, there always exists an adversarial strategy to pick the sequence of loss matrices $G_1,\ldots,G_T$ and a comparator $X$, such that any online algorithm with $\reg(X)=o(T)$ regret requires at least $\Omega(d^2)$ memory.
\end{theorem}

\begin{proof}[Proof of Theorem~\ref{thm:mem_lower}]
There exists a subset $P=\{X_1, \ldots, X_N\}$ of size $N = \exp(d^2)$ with $X_1,\ldots,X_N\in\calX$ and $\|X_i-X_j\|_1\ge 2^{-3}$ for any $i\neq j$ \cite{haah2016sample}. The loss sequence is constructed as follows. 
\begin{itemize}
    \item Choose $X^*$ uniformly at random from $\{X_1, \ldots, X_N\}$.
    \item If the algorithm commits $X_t$ at time $t$, the adversary constructs the loss function $\ell_{X^*, X_t}(X) = \inner{\sgn(X_t-X^*)}{ X}$.
\end{itemize}
The total regret of the algorithm is larger than \[
\sum_{t}\big(\ell_{X^*, X_t}(X_t) - \ell_{X^*, X_t}(X^*)\big)= \sum_t\braket{\sgn(X_t-X^*), X_t-X^*} = \sum_t \|X_t-X^*\|_1.
\]
Suppose that the algorithm uses $m$ bits of memory and hence has $2^m$ memory states in total. Denote the output matrix of the algorithm at memory state $s$ by $X_s$. For any output matrix $X_s$, there is at most one $X_i$ in the packing net $P$ such that $\|X_s-X_i\|_1< 2^{-4}$. For $m= \mathcal{O}(d^2)$ such that $2^S|P|\le 0.1$, the distance $\|X_s-X^*\|_1\ge 2^{-4}$ for all $s$ with probability at least $1-2^m|P| \ge 0.9$. Then the regret is larger than $2^{-4}T$, which contradicts the sublinear regret assumption. 
\end{proof}

We further generalize our memory overhead bounds in the case when we know $X$ is chosen from a subset $\mathcal{X}'\subseteq\mathcal{X}$ in Appendix~\ref{app:imp_memory_bound}. We show a lower bound given by the packing number of $\mathcal{X}'$ and provide an upper bound characterized by the covering number of $\mathcal{X}'$. When $\mathcal{X}'=\mathcal{X}=\Delta_{d\times d}$, this result reduce to Theorem~\ref{thm:mem_lower} as the packing number for $\Delta_{d\times d}$ is $2^{\Theta(d^2)}$~\cite{haah2016sample}.

\section{Applications in Online Learning of Quantum States}\label{sec:quantum}

In this section, we apply our algorithms for matrix LEA (Theorem~\ref{thm:lea_main}) and its extension to matrix OCO (Corollary~\ref{coro:nonlinear_reg}) to learning quantum states in the online setting (see Definition~\ref{def:online_quantum} for the formal definition).

\vspace{0.5em}\noindent\textbf{Matrix OCO. }To begin with, we note that our proposed algorithm can be further generalized to matrix Online Convex Optimization (OCO) \cite[Section~2.3]{orabona2025modern}. In this setting, after receiving a loss matrix $G_t$ from the adversary, the learner incurs a loss $\ell_t(G_t,X_t)$ which is nonlinear but \emph{convex} with respect to $X_t$. Here, we assume that the derivative of the loss function (with respect to the second argument) is bounded, i.e., $\norm{\nabla \ell_t(G_t,X_t)}_{\op}\leq L$ for some known $L$. The overall performance is again measured by the algorithm's cumulative excess loss against a comparator $X$.
\begin{align*}
\reg_T(X)\defeq \sum_{t=1}^T\ell_t(G_t,X_t)-\sum_{t=1}^T\ell_t(G_t,X).
\end{align*}
We have the following result as a standard generalization of Theorem~\ref{thm:lea_main}.
\begin{corollary}\label{coro:nonlinear_reg}
Let $d\geq 1$ and $L\geq 0$. Consider applying the algorithm from Theorem~\ref{thm:lea_main} to the above matrix OCO problem, with $l\leftarrow L$ and the $t$-th received loss matrix being $\nabla \ell_t(G_t,X_t)$. Then, it guarantees for all $T\geq 1$ and $X\in\Delta_{d\times d}$,
\begin{equation*}
\reg_T(X)=\mathcal{O}\rpar{L\sqrt{T\cdot S(X||d^{-1}I_d)}}.
\end{equation*}
\end{corollary}

The proof is due to the convexity of the losses: by Jensen's inequality,
\begin{equation*}
\ell_t(G_t,X_t)-\ell_t(G_t,X)\leq \inner{\nabla \ell_t(G_t,X_t)}{X_t-X},
\end{equation*}
and the summation of the RHS is bounded by Theorem~\ref{thm:lea_main}. 

\subsection{Online learning of noisy and random quantum states}\label{sec:quantum_noisy_random}

\vspace{0.5em}\noindent\textbf{Noisy states. } First, we consider quantum states corrupted by noises on near-term quantum devices~\cite{preskill2018quantum}. Although noises erase the useful information, which is harmful for computing, they also smooth the spectrum of the target quantum states, which simplifies the learning task and can be exploited by our algorithm (but not by MMWU). Specifically, we consider the local depolarization noise, which is a typical noise model for analyzing near-term quantum computation (see, e.g. Ref.~\cite{aharonov2023polynomial}). 

Given a quantum state $\rho$, the noise corruption can be described as a quantum process (also known as a completely positive trace-preserving map) $\mathcal{N}\colon\calH_d\to\calH_d$. The most standard noise is the depolarization noise. Given a quantum state $\rho\in\Delta_{d\times d}$, the noise acts as
\begin{align*}
\mathcal{D}_{d,\gamma}(\rho)=(1-\gamma)\rho+\gamma\frac{I_d}{d},
\end{align*}
where $\gamma$ is known as the noise rate. The local depolarization noise model $\mathcal{D}_{2,\gamma}^{\otimes n}$ is a direct product of ($2$-dimensional) single-qubit depolarization noise on each of the $n=\log d$ qubits. 

Here, we consider noisy quantum circuits of depth $D$ acting on $n=\log d$ qubits. In the quantum channel picture, the noisy circuit can be represented as
\begin{align*}
\Phi_{\mathcal{C},\gamma}=\mathcal{D}_{2,\gamma}^{\otimes n}\circ\mathcal{C}^{(D)}\circ \mathcal{D}_{2,\gamma}^{\otimes n}\circ\mathcal{C}^{(D-1)}\circ \mathcal{D}_{2,\gamma}^{\otimes n}\circ\cdots\circ\mathcal{D}_{2,\gamma}^{\otimes n}\circ \mathcal{C}^{(1)},
\end{align*}
where $\mathcal{C}^{(D)}\cdots\mathcal{C}^{(1)}$ are layers of unitary quantum gates. We refer to the states of the form $\rho=\Phi_{\mathcal{C},\gamma}(\ket{0}\bra{0})$ when we say $\rho$ is prepared by a noisy quantum circuit of depth $D$ with local depolarization noise at each layer with noise rate $\gamma$. 

\begin{corollary}\label{coro:noisy_state}
Let $d,T\geq 1$ and $l>0$. Assume that the underlying unknown quantum state $\rho$ is corrupted by global depolarization noise of rate $\gamma\in[0,1]$, our algorithm can learn $\rho$ with regret bound $\mathcal{O}(l\sqrt{T\log d(1-\gamma)})$. Moreover, if the $\rho$ is prepared by a noisy quantum circuit of depth $D$ with local depolarization noise at each layer, the regret bound for our algorithm is $\mathcal{O}(l(1-\gamma)^D\sqrt{T\log d})$.
\end{corollary}

\begin{proof}[Proof of Corollary~\ref{coro:noisy_state}]
If the underlying state is a quantum state $\rho$ corrupted by a global depolarization noise of rate $\gamma$:
\begin{align*}
\rho_{\gamma}=\mathcal{D}_{d,\gamma}(\rho)=(1-\gamma)\rho+\gamma\frac{I_d}{d}.
\end{align*}
By the convexity of quantum relative entropy, we immediately have the improved regret bound $\mathcal{O}\rpar{l\sqrt{(1-\gamma)T\log d}}$ using Theorem~\ref{thm:lea_main}.

We then consider the state prepared by a noisy quantum circuit of depth $D$ with local depolarization noise. According to the strong data-processing inequality~\cite{hirche2022contraction,quek2024exponentially}, we have
\begin{align*}
S(\Phi_{\mathcal{C},\gamma}(\ket{0}\bra{0})||d^{-1}I_d)\leq(1-\gamma)^{2D}S(\ket{0}\bra{0}||d^{-1}I_d)=(1-\gamma)^{2D}\log d.
\end{align*}
We then have the improved regret bound $\mathcal{O}(l(1-\gamma)^D\sqrt{T\log d})$ using Theorem~\ref{thm:lea_main}.
\end{proof}

\vspace{0.5em}\noindent\textbf{Random states. }We then focus on online learning of random states. We consider two cases: Haar random states and random product states. Random states are essential resources in pseudorandomness (quantum cryptography)~\cite{ji2018pseudorandom}, where Haar random states or states chosen from ensembles that are close to Haar random are used to encrypt messages with security guarantees arising from the (pseudo)randomness. Recent advances in demonstrating quantum advantage in random sampling tasks~\cite{arute2019quantum,zhong2020quantum} also assume the underlying quantum state is random and satisfies the so-called anti-concentration property. For quantum learning, the widely used randomized quantum benchmarking~\cite{elben2023randomized} protocols also apply (approximate) Haar random rotations or random product rotations before making measurements. Here, we show that compared to MMWU, our algorithm can obtain a better regret bound \emph{in the average case} for the online learning of random quantum states.

Formally, we show the following corollary.

\begin{corollary}\label{coro:random_quantum}
Let $d,T\geq 1$ and $l>0$. We consider online learning of random quantum data.
\begin{itemize}
    \item Assume that the underlying unknown quantum state $\rho$ is a $d$-dimensional subsystem of a $d'$-dimensional Haar random quantum state, our algorithm can achieve regret bound $\mathcal{O}(l\sqrt{Td/d'})$ with high probability if $d\ll d'$.
    \item Assume that the underlying unknown quantum state $\rho$ is chosen from a product distribution $\rho\sim\mathcal{S}^{\otimes n}$ with the Pauli second moment matrix $S$ of bounded operator norm $\norm{S}_\op\leq 1-\eta$, our algorithm can achieve regret bound $\mathcal{O}(l\sqrt{T(1-\eta)\log d})$ with high probability.
\end{itemize}
\end{corollary}

\begin{proof}[Proof of Corollary~\ref{coro:random_quantum}]
We consider the first case. As the state $\rho$ is a $d$-dimensional subsystem of a $d'$-dimensional Haar random quantum state, the average von Neumann entropy of $\rho$ is given by $\log d-\mathcal{O}(d/d')$ when $d\ll d'$ according to the Page formula~\cite{page1993average} (see Section~\ref{sec:prelim_quantum}). As $S(\rho||d^{-1}I_d)=\log d-S(\rho)$ where $S(\rho)$ is the von Neumann entropy, the regret of our algorithm in Corollary~\ref{coro:nonlinear_reg} is $\mathcal{O}(l\sqrt{Td/d'})$ with high probability according to the Markov inequality.

We then consider the second case. Recall from Section~\ref{sec:prelim_quantum} that given a random single-qubit state $\sigma\sim\mathcal{S}$, the Pauli second moment matrix $S$ is defined as
\begin{align*}
S_{i,j}=\mathbb{E}_{\sigma\sim\mathcal{S}}[\tr(P_i\sigma)\tr(P_j\sigma)],\quad i,j=1,2,3.
\end{align*}
Here, we consider a random product state $\rho\sim\mathcal{S}^{\otimes n}$. We can also write a random single-qubit state as $\sigma=\frac12(I+r\cdot\vec{\sigma})\sim\mathcal{S}$. We then have $S=\mathbb{E}_{r\sim\mathcal{S}}[rr^\top]$. As we assume $\norm{S}_{\op}\leq 1-\eta$, we have $\mathbb{E}_{r\sim\mathcal{S}}[\abs{r}^2]\leq 3(1-\eta)$. Note that we also have the quantum relative entropy of $\sigma$ written as
\begin{align*}
S(\sigma||I_2/2)=\ln 2-S(\sigma)\leq\frac{\abs{r}^2}{2\ln 2},
\end{align*}
we have
\begin{align*}
\mathbb{E}_{\sigma\sim\mathcal{S}}[S(\sigma||I_2/2)]\leq\frac{3(1-\eta)}{2\ln 2}.
\end{align*}

Choosing a $n$-qubit state $\rho=\sigma_1\otimes\cdots\otimes \sigma_n\sim\mathcal{S}^{\otimes n}$, we have $S(\rho||I_d/d)=\sum_{i=1}^nS(\sigma_i||I_2/2)$. We thus have the average-case regret bound as
\begin{align*}
\mathbb{E}_{\rho\sim\mathcal{S}^{\otimes n}}[\reg_T(\rho)]\leq O\rpar{l\sqrt{T\mathbb{E}_{\rho\sim\mathcal{S}^{\otimes n}}S(\rho||d^{-1}I_d)}}\leq O\rpar{\sqrt{T(1-\eta)n}},
\end{align*}
which follows from the concavity of the square root function.
\end{proof}

\subsection{Online learning of Gibbs states}\label{sec:quantum_gibbs}

Here, we consider online learning of Gibbs states of the form $\rho_\beta=e^{-\beta H}/\tr(e^{-\beta H})$ of a Hamiltonian $H$ at inverse temperature $\beta$. The Gibbs state tells us what the equilibrium state of the quantum system will be if it interacts with the environment at a particular temperature and reaches thermal equilibrium. It is widely considered in quantum Gibbs sampling~\cite{chowdhury2017quantum,kastoryano2016quantum,chen2025quantum}, which is the backbone of many quantum algorithms such as semidefinite programming solvers~\cite{brandao2017quantum,van2017quantum,brandao2019quantum,brandao2022faster}, quantum annealing~\cite{montanaro2015quantum}, quantum machine learning~\cite{wiebe2014quantum}, and quantum simulations at finite temperature~\cite{motta2020determining}.

We analyze the worst-case and average-case performance guarantees of our algorithm. Before providing the results, we need some results from the random matrix theory~\cite{tropp2015introduction,tropp2018second,bandeira2023matrix,brailovskaya2024universality,anderson2010introduction} to define random Hamiltonians. Here, we consider Wigner's Gaussian unitary ensemble (GUE)~\cite{wigner1958distribution}. A $d\times d$ GUE is a family of complex Hermitian random matrices specified by
\begin{align*}
H_{jj}&=\frac{g_{jj}}{\sqrt{d}},\\
H_{jk}&=\frac{g_{jk}+ig_{jk}'}{\sqrt{2d}},\ \text{for }k>j,
\end{align*}
where $g_{jj},g_{jk},g_{jk}'$ are independent standard Gaussian $\mathcal{N}(0,1)$.  The definition we consider here follows from~\cite{chen2024sparse} and has an additional $1/\sqrt{d}$ normalization factor from the standard definition of the Gaussian unitary ensemble. In the following, we denote a random Hamiltonian chosen from GUE as $\hgue$. We can also write a $H\sim\hgue$ in the Pauli basis as:
\begin{align*}
H=\sum_{P\in\cP_n}\frac{g_P}{d}P
\end{align*}
with each $g_P$ an independent standard Gaussian. We will also need the following fact.
\begin{lemma}[{\cite[Fact 8.13]{chen2025information}}]\label{fact:norm_GUE}
With probability at least $1-\exp(\Theta(n))$ for random Hamiltonian $\hgue$ from GUE, we have $\norm{\hgue}_\op\leq 3$.
\end{lemma}

We also consider the random sparse Pauli string (RSPS) Hamiltonian ensemble $H_{\text{RSPS}}$, where each Hamiltonian is an independent sum of $J$ random Pauli strings with random sign coefficients~\cite{chen2024sparse}:
\begin{align*}
H=\sum_{a=1}^J\frac{r_a}{\sqrt{J}}P_a,
\end{align*}
where $P_a$ is IID chosen from $\cP_n$ and $r_a$ is chosen IID uniformly in $\{+1,-1\}$. The properties of random sparse Pauli string Hamiltonians are similar to $\hgue$. Specifically, we have the following fact
\begin{lemma}[{\cite[Theorem III.1]{chen2024sparse}}]\label{fact:norm_rsps}
When $J\geq \mathcal{O}(n^3/\epsilon^4)$, with probability at least $1-\exp(\Theta(n))$ for random Hamiltonian $H\sim H_{\mathrm{RSPS}}$, we have $\norm{H}_\op\leq 3$.
\end{lemma}

Now, we are ready to present our results in the following corollary.

\begin{corollary}\label{coro:gibbs}
Let $d,T\geq 1$ and $l>0$. For any Hamiltonian $H$ with bounded normalized cumulants $\kappa_k\leq h$ with cumulant defined in Eq.~\eqref{eq:cumulant}, our algorithm achieves an $\mathcal{O}(lh\beta\sqrt{T})$ regret bound when the underlying quantum state $\rho$ is a Gibbs state at inverse temperature $\beta=\mathcal{O}(1)$. Furthermore, if the Hamiltonian is a random Gaussian Hamiltonian or a random sparse Hamiltonian in the Pauli basis, our algorithm achieves an $\mathcal{O}(l\beta\sqrt{T})$ regret bound with high probability when $\beta=\mathcal{O}(n)$.
\end{corollary}

\begin{proof}[Proof of Corollary~\ref{coro:gibbs}]
We start with the worst-case guarantee. Given a Hamiltonian $H$ and the temperature parameter $\beta$. We consider the underlying state to be a Gibbs state $\rho_{\beta}=e^{-\beta H}/Z_\beta$ where $Z_\beta=\tr(e^{-\beta H})$ is the partition function. 

We write the exponential as its power series and take the trace term-by-term
\begin{align*}
e^{-\beta H}=I+\sum_{k=1}^\infty\frac{(-1)^k}{k!}(\beta H)^k,\qquad Z_\beta=d+\sum_{k=1}^\infty\frac{(-1)^k\beta^k}{k!}\tr(H^k).
\end{align*}
For symbolic simplicity, we denote $Z_\beta=d(1+\eps_\beta)$. We thus have
\begin{align*}
\eps_\beta=\sum_{k=1}^\infty\frac{(-1)^k\beta^k}{dk!}\tr(H^k).
\end{align*}
Here, we define the normalized $k$-th order moment of $H$ as $\mu_k:=\frac{1}{d}\tr(H^k)$. Using the standard relation between cumulants and moments via partition sets, we define the $k$-th order normalized cumulant of $H$
\begin{align}\label{eq:cumulant}
\kappa_k=\sum_{\pi\in\Pi(k)} (|\pi|-1)!(-1)^{|\pi|-1}\prod_{B\in\pi}\mu_{|B|}.
\end{align}
This leads to the following power series of the free energy:
\begin{align*}
\log Z_\beta=\log d+K(\beta)=\log d+\sum_{n=1}^{\infty}\frac{(-\beta)^k}{k!}\kappa_k
\end{align*}
By thermodynamics, we have $S(\rho_\beta)=\beta\tr(\rho_\beta H)+\log Z_\beta$. We have
\begin{align*}
S(\rho_\beta)=\log d-\sum_{k=2}^{\infty}\frac{(-\beta)^k}{k!}(k-1)\kappa_k
\end{align*}
Therefore, the quantum relative entropy is given by
\begin{align*}
S(\rho_\beta||I_d/d)=\log d-S(\rho_\beta)=\sum_{k=2}^{\infty}\frac{(-\beta)^k}{k!}(k-1)\kappa_k.
\end{align*}
When $\beta=\mathcal{O}(1)$, our online algorithm gives an $\mathcal{O}(lh\beta\sqrt{T})$ regret bound in the \emph{worst case}, which improved over the general bound for any quantum states.

Next, we consider regret bounds in the average case for Gibbs states of random Hamiltonians from $\hgue$ or $H_{\text{RSPS}}$. We denote the Gibbs state of a randomly chosen $\hgue$ ($H_{\text{RSPS}}$) as $\rho_{\beta,\hgue}$ ($\rho_{\beta,H_{\text{RSPS}}}$). We have with probability $1-\exp(-\Theta(n))$
\begin{align*}
S(\rho_{\beta,\hgue}||I_d/d)\leq \mathcal{O}(\beta),\quad S(\rho_{\beta,H_{\text{RSPS}}}||I_d/d)\leq \mathcal{O}(\beta).
\end{align*}
Thus, we have the regret bounded by:
\begin{align*}
\reg_T(\rho_{\beta,\hgue})=\mathcal{O}\rpar{l\sqrt{T\beta}},\quad \reg_T(\rho_{\beta,H_{\text{RSPS}}})=\mathcal{O}\rpar{l\sqrt{T\beta}}
\end{align*}
with probability $1-\exp(-\Theta(n))$.
\end{proof}

\subsection{Online learning of nonlinear properties}
Finally, we consider using our extended algorithm in Corollary~\ref{coro:nonlinear_reg} to predict nonlinear loss functions in quantum information. We consider two loss functions, $\ell_t(O_t,\rho_t)=\tr(O_t\rho_t^2)$ and $\ell_t(O_t,\rho_t)=\tr(O_t\rho_t O_t\rho_t)$. 

The first loss function reduces to purity estimation of the given quantum state in an online setting at $O_t\equiv I$. For a general $O_t$, this task captures quantum virtual cooling~\cite{cotler2019quantum}. These two quantities play an important role in quantum benchmarking~\cite{eisert2020quantum}, experimental and theoretical quantum (entanglement) entropy (purity) estimation~\cite{islam2015measuring,kaufman2016quantum,brydges2019probing,zhang2021experimental,shaw2024benchmarking,gong2024sample,liu2024exponential}, quantum error mitigation~\cite{cai2023quantum,koczor2021exponential,huggins2021virtual}, quantum principal component analysis~\cite{lloyd2014quantum,huang2020predicting,huang2022quantum,liu2024exponential}, and quantum metrology~\cite{giovannetti2011advances}.

The second loss function potentially applies to discovering strong-to-weak spontaneous symmetry breaking (SWSSB) in mixed states~\cite{lessa2025strong}. As $\rho$'s under this scenario are mixed, our algorithm potentially benefits as the spectra of these states are more evenly distributed.

Note that we have $\norm{\nabla\ell_t(O_t,\rho_t)}_{\op}\leq 2\norm{G_t}\leq 2l$ and $\norm{\nabla\ell_t(O_t,\rho_t)}_{\op}\leq 2\norm{G_t}^2\leq 2l^2$ for the two cases, we can obtain the following corollary from Corollary~\ref{coro:nonlinear_reg}.

\begin{corollary}\label{coro:nonlinear_quantum}
Let $d,T\geq 1$ and $l>0$, $O_t\succeq0$, and $\norm{O_t}_\op\leq l$.
\begin{itemize}
    \item When $\ell_t(O_t,\rho_t)=\tr(O_t\rho_t^2)$ (also known as the quantum virtual cooling) , our algorithm achieves an $\mathcal{O}(l\sqrt{T\cdot S(\rho||I_d/d)})$ regret bound.
    \item When $\ell_t(O_t,\rho_t)=\tr(O_t\rho_t O_t\rho_t)$ (also known as the R\'{e}nyi-$2$ correlation function), our algorithm achieves an $\mathcal{O}(l^2\sqrt{T\cdot S(\rho||I_d/d)})$ regret bound.
\end{itemize}
\end{corollary}

\bibliography{Instance-optimal-MMWU}

\appendix
\section{Counterexample for the One-Sided Jensen's Trace Inequality}
\label{app:counterexample}

Here we show that the inequality \eqref{eq:def-jensen} does not hold for the convex function $\Phi(x)=|x|$. We choose $2 \times 2$ symmetric matrices that do not commute. Let $S = \begin{pmatrix} 0 & 1 \\ 1 & 0 \end{pmatrix}$ and $G = \begin{pmatrix} 1 & 0 \\ 0 & -1 \end{pmatrix}$. The operator norm of $G$ is $\|G\|_{\op} = 1$, which allows us to set $l=1$.

First, we evaluate the left-hand side of the inequality, $\tr[|S+G|]$. The sum is $S+G = \begin{pmatrix} 1 & 1 \\ 1 & -1 \end{pmatrix}$. Its characteristic equation is $\det(S+G-\lambda I) = (1-\lambda)(-1-\lambda)-1 = \lambda^2-2=0$, yielding eigenvalues $\lambda_{1,2} = \pm\sqrt{2}$. By the spectral theorem, the matrix $|S+G|$ has eigenvalues $|\lambda_{1,2}| = \sqrt{2}$. The trace, being the sum of the eigenvalues, is therefore
\[
 \tr[|S+G|] = \sqrt{2}+\sqrt{2} = 2\sqrt{2}.
\]

Next, we evaluate the right-hand side, $\tr\spar{\frac{I+G}{2}|S+I|+\frac{I-G}{2}|S-I|}$. The coefficient matrices are $\frac{I+G}{2} = \begin{pmatrix} 1 & 0 \\ 0 & 0 \end{pmatrix}$ and $\frac{I-G}{2} = \begin{pmatrix} 0 & 0 \\ 0 & 1 \end{pmatrix}$. For the absolute value terms, we analyze their spectra. The matrix $S+I = \begin{pmatrix} 1 & 1 \\ 1 & 1 \end{pmatrix}$ has eigenvalues $0$ and $2$. Since it is positive semidefinite, $|S+I| = S+I$. The matrix $S-I = \begin{pmatrix} -1 & 1 \\ 1 & -1 \end{pmatrix}$ has eigenvalues $0$ and $-2$. Thus, $|S-I| = -(S-I) = I-S = \begin{pmatrix} 1 & -1 \\ -1 & 1 \end{pmatrix}$.
Substituting these into the right-hand side expression, the argument of the trace becomes
\begin{align*}
\frac{I+G}{2}|S+I|+\frac{I-G}{2}|S-I| &= \begin{pmatrix} 1 & 0 \\ 0 & 0 \end{pmatrix}\begin{pmatrix} 1 & 1 \\ 1 & 1 \end{pmatrix} + \begin{pmatrix} 0 & 0 \\ 0 & 1 \end{pmatrix}\begin{pmatrix} 1 & -1 \\ -1 & 1 \end{pmatrix} \\
&= \begin{pmatrix} 1 & 1 \\ 0 & 0 \end{pmatrix} + \begin{pmatrix} 0 & 0 \\ -1 & 1 \end{pmatrix} = \begin{pmatrix} 1 & 1 \\ -1 & 1 \end{pmatrix}.
\end{align*}
The trace of this resulting matrix is
\[
\tr\spar{\begin{pmatrix} 1 & 1 \\ -1 & 1 \end{pmatrix}}=2.
\]
Comparing the two sides, we have $\text{LHS} =2\sqrt{2}  > 2= \text{RHS}$, which violates the inequality.

\section{Omitted Proofs for Regret Analysis}

\subsection{Proofs for regret upper bound}\label{section:reg_upper_proofs}

Below are the proofs omitted from Section~\ref{sec:algorithm}.

\reduction*

\begin{proof}[Proof of Lemma~\ref{lemma:reduction}]\label{app:proof_reduction}
The first condition is straightforward: $\norm{\tilde G_t}_\op\leq \norm{\bar G_t}_\op$ since $\tilde G_t$ either equals $\bar G_t$ or is its projection to a subspace. Then, by the triangle inequality, 
\begin{equation*}
\norm{\bar G_t}_{\op}\leq \norm{G_t}_{\op}+\abs{\inner{G_t}{X_t}}\leq \norm{G_t}_{\op}+\norm{G_t}_\op\norm{X_t}_\tr=2\norm{G_t}_{\op}. 
\end{equation*}

To prove the second condition in the lemma, define the intermediate quantity $X_t^+\defeq \sum_{i=1}^d\max\{0,\lambda_{t,i}\} v_{t,i} v_{t,i}^\herm$, which is PSD but not normalized (trace norm not equal to $1$). Due to standard facts in convex optimization \cite[Section~8.1.1]{boyd2004convex}, $X^+_t$ is the projection of $\tilde X_t$ to the PSD cone with respect to the Frobenius norm. By the first order optimality condition, $\inner{X^+_t-\tilde X_t}{X^+_t-X}\leq 0$ for all PSD matrix $X$. Since $U_t$ is the normalized version of $\tilde X_t-X^+_t$, we thus have $\inner{U_t}{X^+_t-X}\geq 0$ for all $X\in\Delta_{d\times d}$. 

The rest of the analysis has two steps. The first step is to show that $\inner{G_t}{X_t-X}= \inner{\bar G_t}{X^+_t-X}$ for all $X\in\Delta_{d\times d}$. This follows from
\begin{align*}
&\inner{G_t}{X_t-X}-\inner{\bar G_t}{X^+_t-X}\\
=~&\inner{G_t}{X_t-X^+_t}+\inner{G_t}{X_t}\inner{I}{X^+_t-X}\tag{definition of $\bar G_t$}\\
=~&\inner{G_t}{\frac{X^+_t}{\norm{X^+_t}_\tr}-X^+_t}+\inner{G_t}{\frac{X^+_t}{\norm{X^+_t}_\tr}}\inner{I}{X^+_t-X}\tag{definition of $X_t$}\\
=~&\inner{G_t}{\frac{X^+_t}{\norm{X^+_t}_\tr}-X^+_t}+\inner{G_t}{\frac{X^+_t}{\norm{X^+_t}_\tr}}\rpar{\norm{X^+_t}_\tr-\norm{X}_\tr}\\
=~&\inner{G_t}{\frac{X^+_t}{\norm{X^+_t}_\tr}-X^+_t}+\inner{G_t}{X^+_t}-\inner{G_t}{\frac{X^+_t}{\norm{X^+_t}_\tr}}\tag{$\norm{X}_\tr=1$}\\
=~&0.
\end{align*}

The second step is to show that $\inner{\bar G_t}{X^+_t-X}\leq \inner{\tilde G_t}{\tilde X_t-X}$ for all $X\in\Delta_{d\times d}$. To this end, there are two cases regarding $\tilde G_t$. 

\vspace{0.5em}\noindent\textbf{Case 1. }If $\tilde G_t=\bar G_t$, then by definition $\inner{\bar G_t}{U_t}\geq 0$. Since $U_t$ is the normalized version 
of $\tilde X_t-X^+_t$, we have $\inner{\bar G_t}{\tilde X_t-X^+_t}\geq 0$, therefore
\begin{equation*}
\inner{\bar G_t}{X^+_t-X}\leq \inner{\bar G_t}{\tilde X_t-X}=\inner{\tilde G_t}{\tilde X_t-X}.
\end{equation*}

\vspace{0.5em}\noindent\textbf{Case 2. }Otherwise, $\tilde G_t=
\bar G_t-\inner{\bar G_t}{U_t}U_t$ and $\inner{\bar G_t}{U_t}\leq 0$. Therefore
\begin{align*}
&\inner{\bar G_t}{X^+_t-X}- \inner{\tilde G_t}{\tilde X_t-X}\\
=~&\inner{\bar G_t}{X^+_t-X}- \inner{\bar G_t-\inner{\bar G_t}{U_t}U_t}{\tilde X_t-X}\tag{definition of $\tilde G_t$}\\
=~&\inner{\bar G_t}{X^+_t-\tilde X_t}+ \inner{\bar G_t}{U_t}\inner{U_t}{\tilde X_t-X}\\
=~&\inner{\bar G_t}{X^+_t-\tilde X_t}+ \inner{\bar G_t}{U_t}\inner{U_t}{\tilde X_t-X^+_t}+\inner{\bar G_t}{U_t}\inner{U_t}{X^+_t-X}\\
\leq~&\inner{\bar G_t}{X^+_t-\tilde X_t}+ \inner{\bar G_t}{U_t}\inner{U_t}{\tilde X_t-X^+_t}\tag{$\inner{\bar G_t}{U_t}\leq 0$, $\inner{U_t}{X^+_t-X}\geq 0$}\\
=~&\inner{\bar G_t}{X^+_t-\tilde X_t}+ \inner{\bar G_t}{\frac{\tilde X_t-X^+_t}{\norm{\tilde X_t-X^+_t}_F}}\inner{\frac{\tilde X_t-X^+_t}{\norm{\tilde X_t-X^+_t}_F}}{\tilde X_t-X^+_t}\tag{definition of $U_t$}\\
=~&0.
\end{align*}

Combining the two steps completes the proof. 
\end{proof}

\expsquare*

\begin{proof}[Proof of Theorem~\ref{thm:lea_worse}]
By Corollary~\ref{thm:phi_itself} and Lemma~\ref{lemma:phi_laplace}, $\Phi^\expsq_t$ satisfies both conditions in Theorem~\ref{thm:olo_master}. Therefore, for the intermediate quantities $\tilde G_t$ and $\tilde X_t$ in Algorithm~\ref{alg:reduction} we have for all comparators $X\in\Delta_{d\times d}$ with eigenvalues $\lambda_1,\ldots,\lambda_d$,
\begin{equation*}
\sum_{t=1}^T\inner{\tilde G_t}{\tilde X_t-X}\leq \inner{\tilde G_1}{\tilde X_1}+\tr\spar{\Phi^\expsq_{1}(-\tilde G_1)}+\sum_{i=1}^d\Phi^{\expsq,*}_T(\lambda_i),
\end{equation*}
where $\Phi^{\expsq,*}_T$ denotes the Fenchel conjugate of $\Phi^{\expsq}_T$. 

$\Phi^{\expsq}_t$ is an even function for all $t$, therefore $\tilde X_1=0\in\mathbb{H}_{d\times d}$. $\norm{\tilde G_1}_\op\leq \eps$ therefore $\tr\spar{\Phi^\expsq_{1}(-\tilde G_1)}\leq d\Phi^\expsq_{1}(\eps)=\sqrt{e}\eps$. The computation of the Fenchel conjugate is due to Ref.~\cite[Lemma~18]{orabona2016coin}: for all $\lambda\geq 0$,
\begin{equation*}
\Phi^{\expsq,*}_T(\lambda)\leq \eps\lambda\rpar{\sqrt{2T\log\rpar{1+\lambda d}}+\sqrt{2T\log T}}.
\end{equation*}
Combining the above and Lemma~\ref{lemma:reduction} leads to the following result: for all $X\in\Delta_{d\times d}$ with eigenvalues $\lambda_1,\ldots,\lambda_d\geq 0$ satisfying $\sum_{i=1}^d\lambda_i=1$, the considered algorithm guarantees
\begin{equation*}
\reg_T(X)\leq 2\sqrt{e}l+2\sqrt{2}l\sqrt{T\log T}+2\sqrt{2}l\sqrt{T}\sum_{i=1}^d\lambda_i\sqrt{\log\rpar{1+\lambda_id}},
\end{equation*}
where
\begin{align*}
\sum_{i=1}^d\lambda_i\sqrt{\log\rpar{1+\lambda_id}}&=\sum_{i=1}^d\sqrt{\lambda_i}\sqrt{\lambda_i\log\rpar{1+\lambda_id}}\\
&\leq \sqrt{\sum_{i=1}^d\lambda_i\log\rpar{1+\lambda_id}}\tag{Cauchy-Schwarz}\\
&\leq \sqrt{1+\sum_{i=1}^d\lambda_i\log\frac{\lambda_i}{d^{-1}}}\tag{$\log(1+x)\leq x^{-1}+\log x$}\\
&=\sqrt{1+S(X||d^{-1}I_d)}.
\end{align*}

The proof is complete by reorganizing the terms.
\end{proof}

\main*

\begin{proof}[Proof of Theorem~\ref{thm:lea_main}]
The proof mirrors that of Theorem~\ref{thm:lea_worse}. For the intermediate quantities $\tilde G_t$ and $\tilde X_t$ in Algorithm~\ref{alg:reduction} we have for all $X\in\Delta_{d\times d}$ with eigenvalues $\lambda_1,\ldots,\lambda_d$,
\begin{equation*}
\sum_{t=1}^T\inner{\tilde G_t}{\tilde X_t-X}\leq \inner{\tilde G_1}{\tilde X_1}+\tr\spar{\Phi^\erfi_{1}(-\tilde G_1)}+\sum_{i=1}^d\Phi^{\erfi,*}_T(\lambda_i),
\end{equation*}
where $\Phi^{\erfi,*}_T$ denotes the Fenchel conjugate of $\Phi^{\erfi}_T$. 

$\tilde X_1=0\in\mathbb{H}_{d\times d}$, and $\tr\spar{\Phi^\erfi_{1}(-\tilde G_1)}\leq d\Phi^\erfi_{1}(\eps)\leq 0$. The computation of the Fenchel conjugate is due to Ref.~\cite[Theorem~4]{zhang2022pde}: for all $\lambda\geq 0$,
\begin{equation*}
\Phi^{\erfi,*}_T(\lambda)\leq \eps\sqrt{T}\spar{d^{-1}+\sqrt{2}\lambda\rpar{\sqrt{\log\rpar{1+\frac{\lambda}{\sqrt{2}d^{-1}}}}+1}}.
\end{equation*}
Combining the above and Lemma~\ref{lemma:reduction}: for all $X\in\Delta_{d\times d}$ with eigenvalues $\lambda_1,\ldots,\lambda_d\geq 0$ satisfying $\sum_{i=1}^d\lambda_i=1$, the considered algorithm guarantees
\begin{equation*}
\reg_T(X)\leq 2(1+\sqrt{2})l\sqrt{T}+2\sqrt{2}l\sqrt{T}\sum_{i=1}^d\lambda_i\sqrt{\log\rpar{1+\frac{\lambda_i}{\sqrt{2}d^{-1}}}},
\end{equation*}
where similar to the proof of Theorem~\ref{thm:lea_worse},
\begin{equation*}
\sum_{i=1}^d\lambda_i\sqrt{\log\rpar{1+\frac{\lambda_i}{\sqrt{2}d^{-1}}}}\leq\sqrt{2+S(X||d^{-1}I_d)}.\qedhere
\end{equation*}
\end{proof}

\subsection{Proofs for regret lower bound}\label{section:regret_lower_proofs}

Below are the proofs omitted from Section~\ref{sec:regret_lower}. We first summarize the Berry-Esseen theorem. 

\begin{lemma}[Berry-Esseen theorem]\label{lemma:berry_esseen}
Let $X_1,\ldots,X_n$ be IID random variables satisfying $\E[X_1]=0$, $\E\spar{X_1^2}=\sigma^2>0$ and $\E\spar{\abs{X_1}^3}=\rho<\infty$. Let $Y_n=\frac{1}{n}\sum_{i=1}^nX_i$, and let $F_n$ be the CDF of $\frac{Y_n\sqrt{n}}{\sigma}$. Then, for all $x$ and $n$, we have
\begin{equation*}
\abs{F_n(x)-\Phi(x)}\leq\frac{\rho}{2\sigma^3\sqrt{n}},
\end{equation*}
where $\Phi$ is the standard normal CDF. 
\end{lemma}

Next is the regret lower bound, based on the anti-concentration result from Lemma~\ref{lemma:anti_concentration}.

\reglower*

\begin{proof}[Proof of Theorem~\ref{thm:reg_lower}]
Let $\Delta^r_{d\times d}\subseteq \Delta_{d\times d}$ be the collection of all $X$ satisfying the constraint $S(X||d^{-1}I_d)\leq r$, and similarly, let $\Delta^r_d\subseteq\Delta_d$ be the collection of all probability vectors $u$ satisfying $\kl(u||d^{-1}\bm{1}_d)\leq r$. Consider a randomized matrix LEA adversary whose output $G_t\in\mathbb{H}_{d\times d}$ is diagonal, and each diagonal entry of $G_t$ is sampled independently from $\mathrm{Uniform}([-1,1])$. For the regret of any algorithm $\calA$, we have
\begin{align*}
\E\spar{\max_{X\in\Delta^r_{d\times d}}\reg_T(X)}&=\sum_{t=1}^T\E\spar{\inner{G_t}{X_t}}+\E\spar{\max_{X\in\Delta^r_{d\times d}}\sum_{t=1}^T\inner{-G_t}{X}}\\
&=\E\spar{\max_{X\in\Delta^r_{d\times d}}\inner{-\sum_{t=1}^TG_t}{X}}\\
&\geq\E\Bigg[\max_{u\in\Delta^r_{d}}\inner{\underbrace{\mathrm{diag}\rpar{-\sum_{t=1}^TG_t}}_{\eqdef Y_T\in\R^d}}{u}\Bigg],\tag{consider diagonal $X$}
\end{align*}
where all expectations are with respect to the randomness of the adversary. Therefore it suffices to lower-bound the RHS, and after that, there exists a deterministic adversary and a comparator $X\in\Delta^r_{d\times d}$ inducing a regret at least this value. 

Next, we collect some basic facts: for $Z \sim \mathrm{Uniform}([-1, 1])$, we have $\E[Z]=0$, $\E[Z^2]=\frac{1}{3}$, and $\E[|Z^3|]=\frac{1}{4}$. For the vector $Y_T$ defined above, each coordinate is the sum of $T$ independent copies of $Z$, i.e., each coordinate follows a centered Irwin-Hall distribution which is unimodal in the sense of Ref.~\cite{ali1965some,david2004order}: its cumulative distribution function is convex on $\R_{<0}$, and concave on $\R_{>0}$. It means Lemma~\ref{lemma:anti_concentration} can be applied to characterize the anti-concentration of $Y_T$. 

Now we pick $u\in\Delta^r_d$ in a sample-dependent manner. Define a constant $k\defeq\lceil d \exp(-r) \rceil$, and let $u_k(Y_T)\in\Delta_d$ be the probability vector that places uniform mass on the indices of the top-$k$ largest entries of $Y_T$ (and zero mass elsewhere). Notice that in particular, $u_k(Y_T)\in\Delta^r_d$ as $\kl(u_k(Y_T)\|d^{-1}\bm{1}_d)=\log(d/k) \le r$. With $Y_{T,(j)}$ being the $j$-th largest entry of the vector $Y_T$, we then have
\begin{align*}
\E\spar{\max_{u\in\Delta^r_d}\inner{Y_T}{u}} \geq \E\spar{\inner{Y_T}{u_k(Y_T)}} = \E\spar{\frac{1}{k}\sum_{j=1}^k Y_{T,(j)}} = \frac{1}{k}\sum_{j=1}^k \E\spar{Y_{T,(j)}}.
\end{align*}

Observe that $r$ and $n$ larger than certain absolute constants would satisfy the requirement of Lemma~\ref{lemma:anti_concentration}, therefore applying it yields
\begin{align*}
\E\spar{\max_{u\in\Delta^c_d}\inner{Y_T}{u}}&\geq \sqrt{\frac{T}{3}}\spar{\sqrt{2\log\frac{d}{\sqrt{2\pi}(d\exp(-r)+2)}}-1}\\
&\geq \sqrt{\frac{T}{3}}\spar{\sqrt{2\rpar{\log\frac{1}{\sqrt{2\pi}\exp(-r)}}-\frac{2}{\sqrt{2\pi}\exp(-r)}\frac{2}{d}}-1}\tag{$\log\frac{1}{x}$ is convex}\\
&\geq \sqrt{\frac{T}{3}}\spar{\sqrt{2r-2\log(2\pi)-\frac{2\sqrt{2}}{\sqrt{\pi}}}-1},\tag{$r\leq\log d$}
\end{align*}
where the terms under the square root is positive for $r$ larger than some absolute constant.  
\end{proof}

\anti*

\begin{proof}[Proof of Lemma~\ref{lemma:anti_concentration}]
By definition, $\calD_n$ is symmetric and unimodal. Let $F_n$ be the CDF of $\calD_n$, and let $F_n^{-1}$ be its inverse which is convex on $\spar{\frac{1}{2},1}$. 

A basic result in order statistics \cite[Section~4.5, Eq.(4.5.9)]{david2004order} states the following: if a distribution with CDF $F$ is symmetric and unimodal, then within IID samples of size $d$, for any $j \geq \frac{d+1}{2}$ we have
\begin{equation*}
\E\spar{Z_{(j)}}\geq F^{-1}\left(\frac{j}{d+1}\right).
\end{equation*}
We apply this to $F_n$ as the index $j$ we are interested in is large enough. Since $F_n^{-1}$ is convex on $\spar{\frac{1}{2},1}$, we further obtain by Jensen's inequality,
\begin{equation}\label{eq:cdf_inverse}
\frac{1}{k}\sum_{j=d-k+1}^{d} \E\spar{Z_{(j)}} \geq  F_n^{-1}\left(\frac{1}{k}\sum_{j=d-k+1}^{d} \frac{j}{d+1}\right) = F_n^{-1}\left(1-\frac{k+1}{2(d+1)}\right).
\end{equation}

Next we use the Berry-Esseen theorem (Lemma~\ref{lemma:berry_esseen}) to relate $F_n$ to the CDF $\Phi$ of the standard normal distribution. For all $x\in\R$,
\begin{equation*}
\abs{\P_{Z\sim\calD_n}\rpar{\frac{Z}{\sigma\sqrt{n}}\leq x}-\Phi(x)}\leq\frac{\rho}{2\sigma^3\sqrt{n}},
\end{equation*}
therefore for all $x\in\R$,
\begin{equation*}
F_n(x)=\P_{Z\sim\calD_n}\rpar{Z\leq x}\leq\Phi\rpar{\frac{x}{\sigma\sqrt{n}}}+\underbrace{\frac{\rho}{2\sigma^3}}_{\eqdef\gamma}\frac{1}{\sqrt{n}}.
\end{equation*}

To proceed, we use classical estimates on the Gaussian tail. Let $\phi$ be the probability density function (PDF) of the standard normal distribution, and consider $\frac{1-\Phi(x)}{\phi(x)}$ which is called the Mills ratio. It is known \cite{gordon1941values} that $\frac{1-\Phi(x)}{\phi(x)}\geq \frac{x}{x^2+1}$ for all $x\geq 0$. Since $\frac{x}{x^2+1}\geq \exp(-x-\frac{1}{2})$ for all $x\geq 1$, we further have for such $x$,
\begin{equation*}
1-\Phi(x)\geq \frac{1}{\sqrt{2\pi}}\exp\rpar{-\frac{1}{2}(x+1)^2}.
\end{equation*}
Combining it with the above, we get for all $x\geq \sigma\sqrt{n}$, 
\begin{equation*}
1-F_n(x)\geq 1-\Phi\rpar{\frac{x}{\sigma\sqrt{n}}}-\frac{\gamma}{\sqrt{n}}\geq \frac{1}{\sqrt{2\pi}}\exp\rpar{-\frac{1}{2}\rpar{\frac{x}{\sigma\sqrt{n}}+1}^2}-\frac{\gamma}{\sqrt{n}}.
\end{equation*}
Inverting this inequality gives a lower bound on $F_n^{-1}$. For any $y \leq \frac{1}{\sqrt{2\pi}e^2} - \frac{\gamma}{\sqrt{n}}$,
\begin{equation*}
F_n^{-1}(1-y)\geq \sigma\sqrt{n}\spar{\sqrt{2\log\frac{1}{\sqrt{2\pi}\left(y+\gamma n^{-1/2}\right)}}-1}.
\end{equation*}

By our assumption on $k$, $\frac{k+1}{2(d+1)}\leq \frac{1}{2\sqrt{2\pi}e^2}$. Furthermore, our assumption on $n$ yields $\frac{\gamma}{\sqrt{n}}\leq\frac{1}{2(d+1)}\leq \frac{k+1}{2(d+1)}$. Therefore we combine the above with Eq.~\eqref{eq:cdf_inverse} to obtain
\begin{align*}
\frac{1}{k}\sum_{j=d-k+1}^{d} \E\spar{Z_{(j)}} &\geq \sigma\sqrt{n}\spar{\sqrt{2\log\frac{1}{\sqrt{2\pi}\left(\frac{k+1}{2(d+1)}+\frac{\gamma}{\sqrt{n}}\right)}}-1}\\
&\geq \sigma\sqrt{n}\spar{\sqrt{2\log\frac{d+1}{\sqrt{2\pi}(k+1)}}-1}\tag{$\frac{\gamma}{\sqrt{n}}\leq \frac{k+1}{2(d+1)}$}\\
&\geq \sigma\sqrt{n}\spar{\sqrt{2\log\frac{d}{\sqrt{2\pi}(k+1)}}-1}.\qedhere
\end{align*}
\end{proof}

\section{Improved Memory Complexity for Restricted Classes of Matrices}\label{app:imp_memory_bound}
In Section~\ref{sec:mem_lower}, we show that the scaling of the memory overhead lower bound is decided by the logarithm in the packing number of $\mathcal{X}'$ while there exists an algorithm that uses memory overhead logarithm in the covering number of $\mathcal{X}'\subseteq\mathcal{X}$ to achieve regret sublinear in $T$.

To begin with, we recap the concept of the packing number and the covering number of a set. Given a set $\mathcal{X}'$ with a distance metric $d(\cdot)$, an \emph{$\eps$-packing} $P\subseteq\mathcal{X}'$ satisfies $d(x,x')\geq\eps$ for any pair $x,x'\in\mathcal{X}'$. The \emph{$\eps$-packing number} $\mathcal{P}(\mathcal{X}',d,\eps)$ is defined to be the maximal size of $\eps$-packing. An \emph{$\eps$-covering} $C\subseteq\mathcal{X}'$ satisfies $d(x,C)\leq\eps$ for any $x\in C$. The \emph{$\eps$-covering number} $\mathcal{C}(\mathcal{X}',d,\eps)$ is defined to be the minimal size of $\eps$-covering. In addition, we have the following relationship:
\begin{align*}
\mathcal{P}(\mathcal{X}',d,2\eps)\leq\mathcal{C}(\mathcal{X}',d,\eps)\leq\mathcal{P}(\mathcal{X}',d,\eps).
\end{align*}

In the following, we can generalize Theorem.~\ref{thm:mem_lower} to a generalized lower bound from the packing number. 

\begin{theorem}\label{thm:mem_lower_improved}
For the matrix LEA problem with matrices chosen from a particular set $\mathcal{X}'\subseteq\calX$ and an adversary, any online algorithm achieving regret $\mathcal{O}(\eps T)$ requires memory size of at least $\log(\mathcal{P}(\mathcal{X}',\norm{\cdot}_1,2\eps)/10)$.
\end{theorem}

\begin{proof}[Proof of Theorem~\ref{thm:mem_lower_improved}]
By the definition of $\eps$-packing number, there exists a subset of the $d$-dimensional density matrices $P=\{X_1, \ldots, X_N\}$ of size $N=\mathcal{P}(\mathcal{X}',\norm{\cdot}_1,2\eps)$ such that $\|X_i-X_j\|_1\ge 2\eps$ for any $i\neq j$.
The loss sequence is constructed as follows. 
\begin{itemize}
    \item Choose $X^*$ uniformly at random from $\{X_1, \ldots, X_N\}$.
    \item If the algorithm commits $X_t$ at time $t$, the adversary constructs the loss function $\ell_{X^*, X_t}(X) = \braket{\sgn(X_t-X^*), X}$.
\end{itemize}
The total regret of the algorithm is larger than \[
\sum_{t}\big(\ell_{X^*, X_t}(X_t) - \ell_{X^*, X_t}(X^*)\big)= \sum_t\braket{\sgn(X_t-X^*), X_t-X^*} = \sum_t \|X_t-X^*\|_1.
\]
Suppose that the algorithm uses $m$ bits of memory and hence has $2^m$ memory states in total. Denote the output matrix of the algorithm at memory state $s$ by $X_s$. For any output matrix $X_s$, there is at most one $X_i$ in the packing net $P$ such that $\|X_s-X_i\|_1<\eps$. For $S=\log(\mathcal{P}(\mathcal{X}',\norm{\cdot}_1,2\eps)/10)$ such that $2^m|P'|\le 0.1$, the distance $\|X_s-X^*\|_1\ge \eps$ for all $s$ with probability at least $1-2^m|P| \ge 0.9$. Then the regret is larger than $\eps T$, which contradicts the sublinear regret assumption. 
\end{proof}

On the other hand, we have the following upper bound:

\begin{theorem}
For the matrix LEA problem with matrices chosen from a particular set $\mathcal{X}'\subseteq\calX_d$ and an adversary, there is an online algorithm achieving regret $o(T)$ with memory size $\log(\mathcal{C}(\mathcal{X}',\norm{\cdot}_1,\eps))$.
\end{theorem}

Here, we suppose a query model for the loss function, i.e., the learner can get access to an entry of the loss matrix $G_t$ in each iteration using one query. The above theorem holds as we can simply use memory to store all the hypothesis matrices in the $\eps$-covering. In each iteration, we query the loss function to compute the loss for all hypothesis matrices and leave the ones with a loss smaller than $\eps$.

As applications, we leave the memory upper and lower bounds for a few classes of quantum states (ignore $\eps$-dependence):
\begin{itemize}
    \item Rank $r$ quantum state $\tilde{\Theta}(rd)$~\cite{haah2016sample}.
    \item Quantum states prepared by with $G$ gates: $\tilde{\Theta}(G)$~\cite{zhao2024learning}.
\end{itemize}

\end{document}